\pdfoutput=1

\documentclass[10pt,twocolumn,letterpaper]{article}

\usepackage[pagenumbers]{cvpr}   %

\input{preamble}
\arxivtrue

\definecolor{cvprblue}{rgb}{0.21,0.49,0.74}
\usepackage[pagebackref,breaklinks,colorlinks,allcolors=cvprblue]{hyperref}   %

\title{Are Watermarked Images Editable? \sysname{} for Watermark-Preserving\\Text-Guided Image Editing}

\author{
Xiaodong Wu$^{1}$\quad Qi Li$^{1}$\quad Xiangman Li$^{1}$\quad Zelin Zhang$^{1}$\quad Lingshuang Liu$^{2}$\quad Jianbing Ni$^{1,*}$\\[2pt]
$^{1}$Queen's University, Kingston, ON, Canada\qquad $^{2}$University of Waterloo, Waterloo, ON, Canada\\[2pt]
{\tt\small \{xiaodong.wu, qi.li, xiangman.li, zelin.zhang, jianbing.ni\}@queensu.ca}\\
{\tt\small lingshuang.liu@uwaterloo.ca}
}

\begin{document}
\twocolumn[{\renewcommand\twocolumn[1][]{#1}\maketitle\begin{center}
\captionsetup{type=figure}
\includegraphics[width=0.9\linewidth]{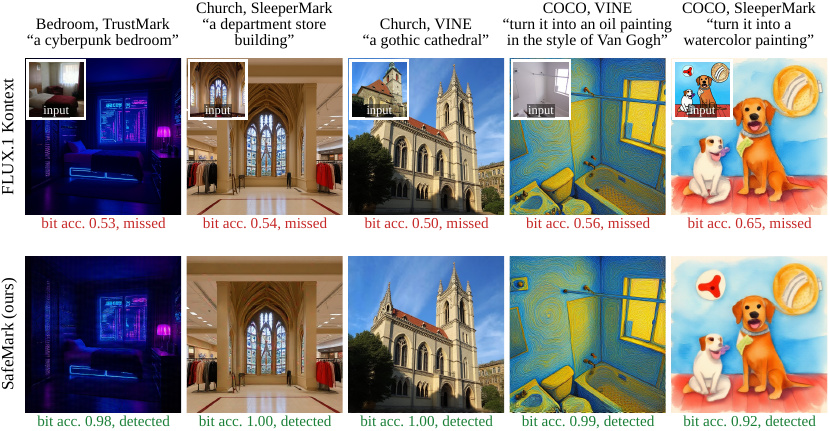}\par
\vspace{4pt}
\captionof{figure}{\textbf{Watermark preservation under FLUX.1 Kontext editing.}
Each watermarked test image (inset) is edited with the same instruction by the original editor (top) and by \sysname (bottom).
Labels give the decoded bit accuracy and whether the watermark is detected at a $1\%$ false-positive rate.
The instructions in the first, second, and fourth columns were not used during fine-tuning.}
\label{fig:teaser}
\end{center}
\vspace{8pt}
}]
{\let\thefootnote\relax\footnotetext{$^{*}$Corresponding author.}}
\begin{abstract}
This paper investigates a fundamental yet underexplored question: \emph{can watermarked images remain editable without compromising watermark integrity?} We propose \textit{SafeMark}, a framework for watermark-preserving text-guided image manipulation that explicitly integrates watermark integrity into the editing process. Specifically, \textit{SafeMark} adds a thresholded watermark-decoding loss directly to the diffusion editor's training objective, fine-tuning the editor so that semantically valid edits also preserve the embedded watermark at the final output. This design admits a clean information-theoretic justification: maintaining high bit-accuracy on the edited image lower-bounds the mutual information that the editor channel preserves between watermark and edited output, the quantity that fundamentally controls watermark recoverability. \textit{SafeMark} is compatible with differentiable diffusion-based editors, requires no architectural modification, and needs neither the watermark decoder nor the message at deployment.
On FLUX.1 Kontext, \textit{SafeMark} raises the detection rate of VINE, TrustMark, and SleeperMark watermarks at a $1\%$ false-positive rate from $16$--$75\%$ to $97.8$--$100\%$ while staying close to the original edits, whereas adapting the watermark decoder to edited images does not recover the watermark. Experiments with pixel-space editors and post-edit distortions further show low false-positive rates and largely retained robustness.
These results demonstrate that semantic editability and watermark integrity are fundamentally compatible, enabling trustworthy image provenance in generative editing pipelines.
\end{abstract}

\section{Introduction}\label{sec:intro}

Rapid advances in generative AI, driven by text-to-image models such as Stable Diffusion~\cite{rombach2022high}, Imagen~\cite{saharia2022photorealistic}, and DALL-E~2~\cite{ramesh2022hierarchical}, have greatly accelerated the creation, modification, and redistribution of digital images. This shift makes reliable ownership verification and provenance tracking increasingly critical for creators, platforms, and rights holders. Invisible image watermarking addresses this need by embedding imperceptible signals into images that can later be recovered by dedicated decoders. Prior work has demonstrated the effectiveness of watermarking in text-to-image generation for provenance-preserving image synthesis~\cite{rezaei2024lawa,kim2024wouaf}, spanning both post hoc encoder--decoder schemes, such as HiDDeN~\cite{zhu2018hidden}, MBRS~\cite{jia2021mbrs}, and StegaStamp~\cite{tancik2020stegastamp}, and model-integrated approaches, including Stable Signature~\cite{fernandez2023stable} and Tree-Ring watermarks~\cite{wen2023tree}.

A key requirement of image watermarking is robustness: the embedded watermark must remain verifiable under distortions. Existing methods, such as Robust-Wide~\cite{hu2024robust}, JigMark~\cite{pan2024jigmark}, VINE~\cite{lu2024robust}, EditGuard~\cite{zhang2024editguard}, and SleeperMark~\cite{wang2025sleepermark}, improve survivability by training against predefined distortions or editing behaviors, an adversarial-defense paradigm whose robustness is fundamentally coupled to the specific threat models assumed during training.
However, text-guided diffusion editors, such as InstructPix2Pix~\cite{brooks2023instructpix2pix}, DiffusionCLIP~\cite{kim2022diffusionclip}, and more recently FLUX.1 Kontext~\cite{labs2025flux}, enable rich, semantics-driven edits via iterative denoising, going far beyond traditional pixel-level perturbations.
This shift raises a critical question: \emph{can watermarked images remain editable without compromising watermark integrity?}

In this paper, we study the coexistence of watermarking and editing and propose \textit{SafeMark}, a framework for watermark-preserving, text-guided image manipulation that treats watermark integrity as an explicit objective during editing. Specifically, we consider an image editing service that fine-tunes its own editor with a frozen watermark decoder, so that the deployed editor needs neither the decoder nor the message carried by the input. We study three key research questions about preserving embedded watermarks under text-guided editing.
\textbf{RQ1) Can existing robust watermarking methods preserve watermarks under arbitrary text-guided image edits?}
Our experiments demonstrate that watermark preservation is strongly editor-dependent: at a $1\%$ false-positive rate, SleeperMark remains detectable on $99.9\%$ of Church images edited with DiffusionCLIP, but on only $33.9\%$ of those edited with FLUX.1 Kontext. Fine-tuning the decoder on up to $3{,}200$ FLUX.1 Kontext edits does not recover the lost watermarks either, indicating that the edit erases the watermark information, which no decoder can then restore.
\textbf{RQ2) Can text-guided image editing be made watermark-safe by design?}
\textit{SafeMark} raises detection to $97.8$--$100\%$ under FLUX.1 Kontext and to $100\%$ under DiffusionCLIP and Eff-Diff~\cite{starodubcev2023towards}, while its outputs stay close to the original edits and unwatermarked inputs rarely acquire a watermark.
\textbf{RQ3) Do \textit{SafeMark}-protected images remain robust to subsequent post-processing distortions?}
Under JPEG compression, blurring, and resizing, \textit{SafeMark} outputs largely retain the detection rates of the unedited watermarked inputs, and failures under strong regeneration mirror those of the input watermark itself.

To the best of our knowledge, \textit{SafeMark} is the first framework that enforces watermark preservation within the editing objective itself, providing protection by design. The main contributions of this paper are summarized as follows:

\begin{itemize}
    \item We identify and empirically demonstrate that existing robust watermarking methods are editor-dependent and fail to generalize across diverse text-guided image editing pipelines, and that adapting the watermark decoder to edited images does not recover the lost watermark.

    \item We cast watermark-preserving editing as constrained optimization and provide an information-theoretic foundation, including a Fano-type recovery bound, accuracy-to-MI lemmas, and a conditional finite-step convergence theorem, to justify SafeMark's hinge-penalty objective.

    \item We propose \textit{SafeMark}, a framework for watermark-preserving text-guided image manipulation that integrates watermark supervision into the editing process, achieving high watermark detection on pixel-space diffusion editors and on the flow-matching editor FLUX.1 Kontext while maintaining high-quality manipulation results and low false-positive rates on unwatermarked inputs.

    \item We demonstrate that \textit{SafeMark} largely retains the robustness of the input watermark to conventional post-edit distortions.
\end{itemize}

\section{Related Work}\label{sec:related}

\paragraph{Robust watermarking.}
Image watermarking embeds an imperceptible signal into visual content so that ownership, provenance, or user-level information can be recovered after redistribution. Classical robust watermarking has evolved from hand-crafted frequency-domain designs to neural encoder--decoder systems that are trained against simulated channels such as JPEG compression, resizing, blurring, and color perturbations. Representative methods include HiDDeN \cite{zhu2018hidden}, StegaStamp \cite{tancik2020stegastamp}, MBRS \cite{jia2021mbrs}, RoSteALS \cite{bui2023rosteals}, and TrustMark \cite{bui2023trustmark}. These methods substantially improve robustness to pixel-level and acquisition-time distortions, but text-guided diffusion editing introduces a different challenge: the image may be partially re-synthesized, semantically transformed, or regenerated through a latent generative prior. Recent studies and benchmarks \cite{pan2024jigmark,lu2024robust,an2024waves} show that watermark recovery can degrade sharply under such semantic editing, indicating that robustness to conventional distortions does not directly translate into watermark preservation during generative manipulation.

\paragraph{Editing-aware watermarking.}
This gap has motivated editing-aware watermarking, where watermark training or verification explicitly accounts for modern image editors. Robust-Wide \cite{hu2024robust} injects instruction-driven editing operations into the watermark training loop, encouraging watermark signals to survive semantic changes produced by editors such as InstructPix2Pix \cite{brooks2023instructpix2pix}. JigMark \cite{pan2024jigmark} learns from original/edited image pairs in a black-box manner to improve survival under diffusion edits, while VINE \cite{lu2024robust} analyzes how diffusion editing suppresses high-frequency watermark components and uses generative priors with surrogate degradations to improve edit robustness. Beyond copyright decoding alone, EditGuard \cite{zhang2024editguard} combines watermarking with tamper localization for both classical and AIGC-based edits, and SleeperMark \cite{wang2025sleepermark} studies watermark persistence after fine-tuning text-to-image diffusion models. However, most existing methods still optimize the watermark itself against selected editors, edit distributions, or downstream adaptation settings for robustness, and do not directly constrain the editing process to preserve an already embedded watermark. Our work adapts the editor so that it preserves the embedded watermark while performing the requested edit.

\paragraph{Text-guided image editing.}
Diffusion-based editors change an image according to a text prompt. DiffusionCLIP~\cite{kim2022diffusionclip} and Eff-Diff~\cite{starodubcev2023towards} fine-tune the denoising network for each target attribute, Asyrp~\cite{kwon2022diffusion} edits in the semantic h-space of a frozen network, SDEdit~\cite{meng2021sdedit} and DDIM inversion~\cite{song2021ddim} re-synthesize a noised version of the input under a new prompt, Prompt-to-Prompt~\cite{hertz2022prompt} controls cross-attention maps, and InstructPix2Pix~\cite{brooks2023instructpix2pix} follows free-form instructions. Recent editors such as FLUX.1 Kontext~\cite{labs2025flux} are large transformers trained with flow matching~\cite{lipman2023flow,liu2023flow} that attend to the tokens of the input image while following an instruction. These editors are designed for edit quality and controllability, and none of them accounts for a watermark that the input image already carries.

\section{Problem Formulation and Theoretical Foundations}\label{sec:theory}

\subsection{Problem Setup}\label{sec:problem}

An image $\mathbf{x}_{\text{orig}}$ watermarked with a message $\mathbf{w}\in\{0,1\}^{B}$ is edited by $\mathcal{E}_\theta$ under prompt $\mathbf{p}$, with decoder $\mathcal{D}_\phi$ recovering $\hat{\mathbf{w}}$ from the output. The encoder and decoder are fixed, and we optimize only $\theta$.

\textbf{Constrained Objective.}
Let $\mathbf{x}_{\text{edit}}=\mathcal{E}_\theta(\mathbf{x}_{\text{orig}},\mathbf{p})$.
\textit{SafeMark} seeks $\theta^\star$ satisfying:
\begin{equation}
\label{eq:constrained_opt}
\begin{aligned}
\min_{\theta}\quad &
\mathbb{E}_{\mathbf{x}_{\text{orig}},\mathbf{w},\mathbf{p}}
\Big[\mathcal{L}_{\text{sem}}\big(\mathbf{x}_{\text{edit}},\,\mathbf{x}_{\text{ref}}\big)\Big] \\
\text{s.t.}\quad &
\mathrm{Acc}\Big(\hat{\mathbf{w}},\,\mathbf{w}\Big)\ \ge\ \tau,
\qquad
\hat{\mathbf{w}}=\sigma\!\Big(\mathcal{D}_\phi(\mathbf{x}_{\text{edit}})\Big),
\end{aligned}
\end{equation}
where $\tau\in[0,1]$ is the accuracy threshold and $\mathrm{Acc}$ is the hard-bit accuracy reported at evaluation. The semantic loss $\mathcal{L}_{\text{sem}}$ is instantiated in Section~\ref{sec:method}. During training, the non-differentiable $\mathrm{Acc}$ is replaced by a smooth surrogate $\widetilde{\mathrm{Acc}}$ derived in Section~\ref{sec:penalty}.

\begin{figure*}[t]
    \centering
    \includegraphics[width=0.9\textwidth]{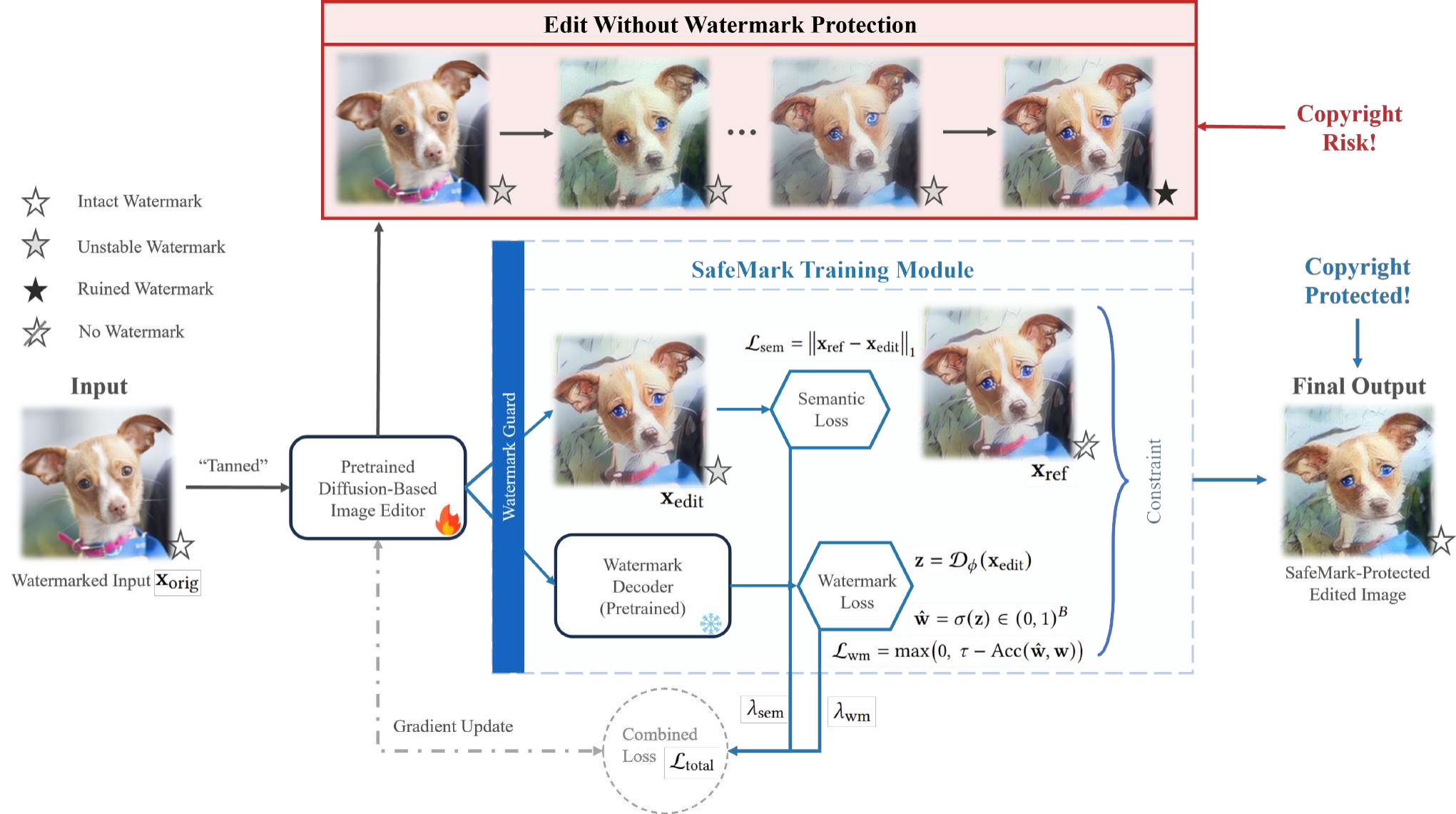}
    \caption{\textbf{Overview of \textit{SafeMark}.}
    Given a watermarked input image $\mathbf{x}_{\text{orig}}$ and an editing prompt $\mathbf{p}$, the diffusion editor produces the edited image $\mathbf{x}_{\text{edit}}$.
    \textit{SafeMark} enforces (i) \emph{semantic alignment} to a reference edit $\mathbf{x}_{\text{ref}}=\mathcal{E}_{\theta_0}(\mathbf{x}_{\text{orig}},\mathbf{p})$ and
    (ii) \emph{watermark preservation} by decoding $\hat{\mathbf{w}}=\sigma(\mathcal{D}_\phi(\mathbf{x}_{\text{edit}}))$ and applying a thresholded loss when decoding accuracy drops below $\tau$.
    }
    \label{fig:safemark_pipeline}
\end{figure*}

\subsection{Information-Theoretic Foundation}\label{sec:capacity}

Watermark recoverability is bounded by $I_\theta(W;X_e)$, the mutual information preserved by the editor channel. In this view, hard-bit accuracy provably lower-bounds this quantity, thereby grounding \textit{SafeMark}'s bit-accuracy-driven hinge training objective.

\paragraph{Editor as a watermark channel.}
We treat $W\in\{0,1\}^B$ as uniform with independent bits, matching standard encoders (HiDDeN~\cite{zhu2018hidden}, VINE~\cite{lu2024robust}, etc.). The editor produces $X_e=\mathcal{E}_\theta(X_{\text{orig}},\mathbf{p};\xi)$, and the chain $W\to X_{\text{orig}}\to X_e\to \hat W$ is Markov. We call the conditional distribution of $X_e$ given $W$ the \emph{editor watermark channel}, with mutual information $I_\theta(W;X_e)$ measured in bits. Training uses a fixed reference message, for which these bounds give no guarantee by themselves. Section~\ref{sec:eval_rq2} therefore evaluates the trained editor on randomly sampled messages.

\begin{theorem}[Information-theoretic necessary condition for watermark recovery]
\label{thm:fano}
For any editor parameter $\theta$ and decoder $D_\phi$, the block error rate $P_e(\theta,\phi):=\Pr[\hat W\neq W]$ satisfies
\begin{equation}
\label{eq:fano}
P_e(\theta,\phi)
\;\ge\;
1-\frac{I_\theta(W;X_e)+1}{H(W)}.
\end{equation}
With $W$ uniform on $\{0,1\}^B$, $H(W)=B$ bits, so $P_e\ge 1-(I_\theta(W;X_e)+1)/B$. \emph{(Proof in Supp.~\ref{app:theory}.)}
\end{theorem}

For a fixed watermark encoder and message distribution, optimizing the decoder cannot exceed what $I_\theta$ already permits. \textit{SafeMark} therefore reparameterizes $\theta$ to raise $I_\theta$ directly.

\paragraph{From bit accuracy to mutual information.}
Let $\bar a:=\tfrac{1}{B}\sum_b\Pr[\hat W_b=W_b]$ denote average bit accuracy.

\begin{lemma}[Bit accuracy lower-bounds channel mutual information]
\label{lem:acc-mi}
Under the assumed independence of the bits of $W$,
\begin{equation}
\label{eq:acc-mi}
I_\theta(W;X_e)
\;\ge\;
H(W)\;-\;B\,H_2\!\left(1-\bar a(\theta,\phi)\right).
\end{equation}
With $W$ uniform on $\{0,1\}^B$, this becomes $I_\theta(W;X_e)\ge B\bigl(1-H_2(1-\bar a)\bigr)$ bits. \emph{(Proof in Supp.~\ref{app:theory}.)}
\end{lemma}

\paragraph{Bit accuracy as a tractable surrogate target.}
Enforcing $\bar a(\theta,\phi)\ge\tau$ implies
\begin{equation}
\label{eq:capacity-from-tau}
I_\theta(W;X_e)\;\ge\;B\bigl(1-H_2(1-\tau)\bigr)\ \text{bits},
\end{equation}
motivating bit-level training as a tractable proxy for the mutual information (MI) objective. For $\tau\in[1/2,1]$, the bound is monotone increasing in $\tau$ and saturates at the maximum $B$ bits when $\tau=1$.

\subsection{Hinge Penalty Relaxation}\label{sec:penalty}

Since the feasible set is defined through a non-linear multi-step denoising trajectory, solving Eq.~\eqref{eq:constrained_opt} directly is computationally prohibitive. \textit{SafeMark} instead relaxes the constraint via a one-sided hinge penalty on a Brier-type differentiable surrogate of hard bit accuracy:
\begin{equation}
\label{eq:soft-acc}
\widetilde{\mathrm{Acc}}(\hat{\mathbf{w}},\mathbf{w})\;:=\;\frac{1}{B}\sum_{b=1}^{B}\!\Big(1-(\hat w_b-w_b)^2\Big),
\qquad \hat w_b\in[0,1],
\end{equation}
where $\hat w_b=\sigma(\mathcal{D}_\phi(\mathbf{x}_{\text{edit}}))_b$. Combining with the semantic loss $\mathcal{L}_{\text{sem}}$ (instantiated in Section~\ref{sec:method}), the training objective is
\begin{align}
\mathcal{L}_{\text{total}}(\theta)
&= \lambda_{\text{sem}}\mathcal{L}_{\text{sem}} + \lambda_{\text{wm}}\mathcal{L}_{\text{wm}}, \label{eq:total} \\
\mathcal{L}_{\text{wm}}
&= \max\bigl(0,\;\tau-\widetilde{\mathrm{Acc}}(\hat{\mathbf{w}},\mathbf{w})\bigr). \label{eq:wm_hinge}
\end{align}

\paragraph{Calibration: from soft surrogate to hard accuracy.}
By a pointwise wrong-bit-count argument (Lemma~\ref{lem:calibration}, Supp.~\ref{app:theory}), training to $\widetilde{\mathrm{Acc}}\ge\tau$ implies hard-bit accuracy $\bar a\ge 4\tau-3$. Combining this bound with Lemma~\ref{lem:acc-mi} yields
\begin{equation}
\label{eq:soft-mi}
I_\theta(W;X_e)\;\ge\;B\bigl(1-H_2\bigl(\,4(1-\tau)\,\bigr)\bigr),\quad\text{whenever}\quad \tau\ge 7/8,
\end{equation}
so the soft hinge is the variational counterpart of Eq.~\eqref{eq:capacity-from-tau} for $\tau\ge 7/8$. At $\tau=1$, zero hinge loss implies perfect hard-bit recovery.

Once $\widetilde{\mathrm{Acc}}\ge\tau$ (inactive phase), the hinge contributes zero gradient and training reduces to $\lambda_{\text{sem}}\mathcal{L}_{\text{sem}}$ alone. Under explicit non-vanishing-gradient and gradient-alignment assumptions on the active phase, gradient descent enters the inactive phase in finitely many steps (Theorem~\ref{thm:recovery}, Supp.~\ref{app:theory}). Since sigmoid outputs keep $\widetilde{\mathrm{Acc}}<1$, the hinge remains active throughout training at the default $\tau=1$ used in our experiments, where it acts as a mean squared bit error. The theorem therefore concerns $\tau<1$.

\section{SafeMark}\label{sec:safemark}

\subsection{Methodology}
\label{sec:method}

The editor $\mathcal{E}_\theta$, decoder $\mathcal{D}_\phi$, edited output $\mathbf{x}_{\text{edit}}$, soft accuracy $\widetilde{\mathrm{Acc}}$, and hinge loss $\mathcal{L}_{\text{wm}}$ are as defined in Sections~\ref{sec:problem} and~\ref{sec:penalty}. Here we specify the SafeMark-specific instantiation. An overview of the pipeline is illustrated in Fig.~\ref{fig:safemark_pipeline}.

\paragraph{Reference target for semantic supervision.}
To provide a stable semantic target without over-constraining the optimization, we use the \emph{frozen} pre-trained editor $\mathcal{E}_{\theta_0}$ (i.e., without our training updates) to produce a reference edit $\mathbf{x}_{\mathrm{ref}}=\mathcal{E}_{\theta_0}(\mathbf{x}_{\mathrm{orig}},\mathbf{p})$, capturing the intended editing direction.
The reference and the trainable editor receive the same input, prompt, and sampling noise, so that the reference isolates the requested edit from sampling randomness.

\paragraph{Semantic alignment loss.}
Semantic consistency is enforced only on the final edited image by encouraging $\mathbf{x}_{\text{edit}}$ to remain close to the reference:
\begin{equation}
\label{eq:l1-sem}
\mathcal{L}_{\text{sem}}\;=\;\|\mathbf{x}_{\text{ref}}-\mathbf{x}_{\text{edit}}\|_{1}.
\end{equation}

\paragraph{Watermark decoding and soft accuracy.}
The frozen decoder $\mathcal{D}_\phi$ is applied to the final edited image $\mathbf{x}_{\text{edit}}$, and a sigmoid activation yields soft bit estimates $\hat{\mathbf{w}}=\sigma(\mathcal{D}_\phi(\mathbf{x}_{\text{edit}}))\in[0,1]^B$.
Since hard bit accuracy is non-differentiable, we adopt the Brier-type surrogate $\widetilde{\mathrm{Acc}}$ (Eq.~\eqref{eq:soft-acc}), which is smooth and admits gradient flow from the decoded bits back through the editor.

\paragraph{Overall training objective.}
The full objective is $\mathcal{L}_{\text{total}}=\lambda_{\text{sem}}\mathcal{L}_{\text{sem}}+\lambda_{\text{wm}}\mathcal{L}_{\text{wm}}$ as in Eq.~\eqref{eq:total}, with $\widetilde{\mathrm{Acc}}$ feeding the hinge $\mathcal{L}_{\text{wm}}$ in Eq.~\eqref{eq:wm_hinge}. The weight $\lambda_{\text{sem}}$ sets how strongly the adapted editor is anchored to the original editing behavior. \Cref{app:kontext_details} lists the weights we use, and \cref{tab:design,tab:trust_lambda,tab:tau} analyze their effect and the sensitivity to $\tau$. Gradients flow only through the final output $\mathbf{x}_{\text{edit}}$, with no supervision on intermediate denoising steps.

\subsection{Training Procedure}\label{sec:training}

The training procedure of \textit{SafeMark} is summarized in Algorithm~\ref{alg:safemark}.
For each minibatch, we first compute a fixed semantic reference $\mathbf{x}_{\text{ref}} = \mathcal{E}_{\theta_0}(\mathbf{x}_{\text{orig}}, \mathbf{p}_{\text{txt}})$ using the \emph{frozen} editor $\mathcal{E}_{\theta_0}$ (line~2), capturing the intended editing direction.
The trainable editor $\mathcal{E}_{\theta}$ then produces the final edited image $\mathbf{x}_{\text{edit}}$ without supervision at intermediate steps (line~3).
Semantic alignment encourages $\mathbf{x}_{\text{edit}}$ to remain close to the reference (line~4), while watermark supervision decodes the embedded watermark and computes bit accuracy (lines~5--6).
A thresholded hinge loss penalizes accuracy below $\tau$ (line~7), and the combined objective updates the editor parameters (lines~8--9).
Because the watermark loss is applied only to the final output $\mathbf{x}_{\text{edit}}$ and gradients are propagated through it, \textit{SafeMark} is compatible with differentiable diffusion-based editors and applies watermark supervision only at the final denoised output, requiring no architectural modification to the underlying editor.
At inference time, only the fine-tuned editor $\mathcal{E}_\theta$ is required. The frozen reference editor $\mathcal{E}_{\theta_0}$ and watermark decoder $\mathcal{D}_\phi$ are used exclusively during training and introduce no additional computational overhead at deployment.

\begin{algorithm}[t]
\caption{\textit{SafeMark} Training Procedure}
\label{alg:safemark}
\small
\begin{algorithmic}[1]
\Require Trainable editor $\mathcal{E}_{\theta}$, frozen editor $\mathcal{E}_{\theta_0}$, decoder $\mathcal{D}_{\phi}$, triples $(\mathbf{x}_{\rm orig},\mathbf{w},\mathbf{p}_{\rm txt})$, weights $\lambda_{\rm sem},\lambda_{\rm wm}$, threshold $\tau$, step size $\eta$
\For{each minibatch $(\mathbf{x}_{\rm orig},\mathbf{w},\mathbf{p}_{\rm txt})$}
    \State $\mathbf{x}_{\rm ref}\gets\mathcal{E}_{\theta_0}(\mathbf{x}_{\rm orig},\mathbf{p}_{\rm txt})$
    \State $\mathbf{x}_{\rm edit}\gets\mathcal{E}_{\theta}(\mathbf{x}_{\rm orig},\mathbf{p}_{\rm txt})$
    \State $\mathcal{L}_{\rm sem}\gets\|\mathbf{x}_{\rm ref}-\mathbf{x}_{\rm edit}\|_1$
    \State $\hat{\mathbf{w}}\gets\sigma(\mathcal{D}_{\phi}(\mathbf{x}_{\rm edit}))$
    \State $\widetilde{\mathrm{Acc}}\gets B^{-1}\sum_{b=1}^{B}\!\left[1-(\hat{w}_b-w_b)^2\right]$
    \State $\mathcal{L}_{\rm wm}\gets\max(0,\tau-\widetilde{\mathrm{Acc}})$
    \State $\mathcal{L}_{\rm total}\gets\lambda_{\rm sem}\mathcal{L}_{\rm sem}+\lambda_{\rm wm}\mathcal{L}_{\rm wm}$
    \State $\theta\gets\theta-\eta\nabla_{\theta}\mathcal{L}_{\rm total}$
\EndFor
\end{algorithmic}
\end{algorithm}

\begin{table*}[t]
\centering
\caption{\textbf{Watermark preservation with FLUX.1 Kontext} on held-out test images (mean $\pm$ standard deviation over three training seeds). Base and SM denote the original editor and \textit{SafeMark}. TPR and FPR are the detection rates for watermarked and unwatermarked inputs at a $1\%$ false-positive threshold. Acc and Own are the bit accuracies for watermarked and random-message inputs. Own requires a watermark encoder and is unavailable for SleeperMark. LPIPS, PSNR, and $\Delta$CLIP-T compare SM with Base outputs.}
\label{tab:kontext}
\small
\setlength{\tabcolsep}{0pt}
\begin{tabular*}{\textwidth}{@{\extracolsep{\fill}}llcccccccccc@{}}
\toprule
\multirow{2}{*}{Watermark} & \multirow{2}{*}{Data} & \multicolumn{2}{c}{TPR (\%) $\uparrow$} & \multicolumn{2}{c}{FPR (\%) $\downarrow$} & \multicolumn{2}{c}{Acc $\uparrow$} & \multirow{2}{*}{Own $\uparrow$} & \multirow{2}{*}{LPIPS $\downarrow$} & \multirow{2}{*}{PSNR $\uparrow$} & \multirow{2}{*}{$\Delta$CLIP-T} \\
\cmidrule(lr){3-4} \cmidrule(lr){5-6} \cmidrule(lr){7-8}
 & & Base & SM & Base & SM & Base & SM & & & & \\
\midrule
VINE & Bedroom & 23.9 & \textbf{100.0{\scriptsize$\pm$0.0}} & 0.0 & 0.0 & 0.599 & 0.990 & 0.992 & 0.021{\scriptsize$\pm$0.000} & 36.9 & $-$0.003 \\
 & Church & 42.2 & \textbf{100.0{\scriptsize$\pm$0.0}} & 2.8 & 0.0 & 0.687 & 0.973 & 0.982 & 0.020{\scriptsize$\pm$0.006} & 33.7 & $+$0.002 \\
 & COCO & 65.0 & \textbf{100.0{\scriptsize$\pm$0.0}} & 0.0 & 0.0 & 0.810 & 0.988 & 0.981 & 0.058{\scriptsize$\pm$0.020} & 31.1 & $-$0.001 \\
\addlinespace[2pt]
TrustMark & Bedroom & 16.1 & \textbf{100.0{\scriptsize$\pm$0.0}} & 0.6 & 1.1 & 0.577 & 0.970 & 0.960 & 0.032{\scriptsize$\pm$0.016} & 35.8 & $-$0.002 \\
 & Church & 41.7 & \textbf{99.4{\scriptsize$\pm$1.0}} & 0.0 & 1.1 & 0.661 & 0.921 & 0.918 & 0.054{\scriptsize$\pm$0.001} & 30.5 & $+$0.002 \\
 & COCO & 60.6 & \textbf{98.9{\scriptsize$\pm$1.9}} & 0.0 & 1.1 & 0.755 & 0.930 & 0.938 & 0.166{\scriptsize$\pm$0.045} & 27.7 & $-$0.007 \\
\addlinespace[2pt]
SleeperMark & Bedroom & 17.8 & \textbf{99.4{\scriptsize$\pm$1.0}} & 0.6 & 0.0 & 0.569 & 0.952 & -- & 0.019{\scriptsize$\pm$0.005} & 39.0 & $-$0.001 \\
 & Church & 33.9 & \textbf{100.0{\scriptsize$\pm$0.0}} & 0.0 & 0.0 & 0.658 & 0.948 & -- & 0.041{\scriptsize$\pm$0.007} & 31.6 & $+$0.003 \\
 & COCO & 75.0 & \textbf{97.8{\scriptsize$\pm$1.0}} & 0.0 & 0.6 & 0.837 & 0.939 & -- & 0.104{\scriptsize$\pm$0.025} & 29.3 & $-$0.004 \\
\bottomrule
\end{tabular*}
\end{table*}

\subsection{Fine-Tuning Large Flow-Matching Editors}\label{sec:large}

For pixel-space editors such as DiffusionCLIP and Eff-Diff, \textit{SafeMark} fine-tunes the full denoising network. FLUX.1 Kontext~\cite{labs2025flux} is a 12B-parameter rectified flow transformer~\cite{esser2024scaling} that edits an image by appending the tokens of the encoded input to the noisy target tokens and integrating its velocity field over 28 sampling steps. Two design choices make Algorithm~\ref{alg:safemark} tractable for such an editor without changing its objective.

\paragraph{Low-rank adaptation of attention.}
We train rank-16 LoRA adapters~\cite{hu2022lora} on the attention projections of the transformer, 26.1M parameters in total, and keep all original weights frozen. The frozen editor $\mathcal{E}_{\theta_0}$ is then the same network with the adapters disabled, so no second copy of the model is needed.

\paragraph{Late-step adaptation.}
The adapters are active only in the last $k=6$ of the 28 sampling steps, where the noise level $\sigma$, which decreases from $1$ to $0$ during sampling, is at most $0.34$ and the layout of the edit is already determined. The earlier steps are shared by $\mathbf{x}_{\text{ref}}$ and $\mathbf{x}_{\text{edit}}$, and gradients flow through the last $k$ steps and the VAE decoder only. When the adapters were active in all 28 steps, the edited images drifted away from the requested edit (\cref{tab:design}). Late-step adaptation leaves the content of the edit to the original editor and adjusts only fine image details. Because the shared steps and $\mathbf{x}_{\text{ref}}$ do not depend on $\theta$, they can be computed once per training image, prompt, and noise seed before fine-tuning, which we do for runs of 1{,}000 steps.

\section{Evaluation}\label{sec:eva}

\subsection{Experiment Setup}\label{sec:exp_setup}

\noindent\textbf{Editors.}
We evaluate \textit{SafeMark} on three text-guided editors that differ in architecture and editing mechanism. FLUX.1 Kontext [dev]~\cite{labs2025flux}, an instruction-following rectified flow transformer, is our main editor. DiffusionCLIP~\cite{kim2022diffusionclip} and Eff-Diff~\cite{starodubcev2023towards} are pixel-space diffusion editors that are fine-tuned for one target attribute at a time. \cref{app:other_editors} additionally reports Stable Diffusion editors based on DDIM inversion~\cite{song2021ddim} and InstructPix2Pix~\cite{brooks2023instructpix2pix}.

\noindent\textbf{Watermarks and datasets.}
We watermark real images with VINE~\cite{lu2024robust} and TrustMark~\cite{bui2023trustmark}, which embed $100$ bits each, and use SleeperMark~\cite{wang2025sleepermark}, which embeds $48$ bits into images generated by its watermarked text-to-image model. Real images come from LSUN-Bedroom and LSUN-Church~\cite{yu2015lsun} at $256\times256$ and, as general photos, from COCO~\cite{lin2014microsoft} at $512\times512$. SleeperMark images are generated for the same two scene types and from COCO captions. For each editor, dataset, and watermark, we use $50$ training images and three held-out sets of $30$ images each: watermarked test images, unwatermarked images, and, for VINE and TrustMark, images re-watermarked with independently sampled random messages.

\noindent\textbf{Editing instructions.}
For FLUX.1 Kontext, each dataset has $12$ instructions, of which $8$ are used for fine-tuning and $4$ are held out. The LSUN instructions change the scene type (\eg, \emph{a gothic cathedral}), and the COCO instructions combine global changes of style or lighting with local additions (\eg, \emph{add a butterfly to the scene}). Each test image is edited once with a training instruction and once with a held-out instruction. The pixel-space editors are fine-tuned for each target attribute separately, and $N$ in \cref{tab:pixel} counts the attributes. Each FLUX.1 Kontext configuration is trained with three random seeds, and \cref{app:kontext_details} gives the training settings, including the loss weights and the rule that selects $\lambda_{\text{sem}}$ for SleeperMark from training images only.

\noindent\textbf{Metrics.}
Bit accuracy (Acc) is the fraction of correctly decoded bits. An edited image is detected when at least $63$ of $100$ (VINE, TrustMark) or $33$ of $48$ (SleeperMark) decoded bits match the reference message, which corresponds to a $1\%$ false-positive rate for random bits. The true positive rate (TPR) is the percentage of edited test images that are detected, and the false positive rate (FPR) is the percentage of edited unwatermarked images that are detected. For random-message inputs, Own is the bit accuracy with respect to the message embedded in the input. LPIPS~\cite{zhang2018lpips} and PSNR measure how far a \textit{SafeMark} output deviates from the original editor's output for the same input, instruction, and noise, and $\Delta$CLIP-T is the resulting change in CLIP~\cite{radford2021learning} similarity between the output and the instruction.

\begin{table}[t]
\centering
\caption{\textbf{Pixel-space diffusion editors.} Detection and fidelity on held-out test images, aggregated over $N$ target attributes. Columns as in \cref{tab:kontext}.}
\label{tab:pixel}
\footnotesize
\setlength{\tabcolsep}{0pt}
\begin{tabular*}{\columnwidth}{@{\extracolsep{\fill}}lllccccc@{}}
\toprule
\multirow{2}{*}{Editor} & \multirow{2}{*}{WM} & \multirow{2}{*}{Data} & \multicolumn{2}{c}{TPR (\%)} & \multicolumn{2}{c}{FPR (\%)} & \multirow{2}{*}{LPIPS} \\
\cmidrule(lr){4-5} \cmidrule(lr){6-7}
 & & & Base & SM & Base & SM & \\
\midrule
DiffCLIP & VINE & Bedroom & 99.8 & \textbf{100.0} & 0.1 & 0.1 & 0.069 \\
 &  & Church & 100.0 & \textbf{100.0} & 0.4 & 1.2 & 0.060 \\
\addlinespace[1pt]
 & Trust & Bedroom & 95.0 & \textbf{100.0} & 1.0 & 0.7 & 0.056 \\
 &  & Church & 96.0 & \textbf{100.0} & 0.4 & 0.4 & 0.041 \\
\addlinespace[1pt]
 & Sleeper & Bedroom & 99.6 & \textbf{100.0} & 0.0 & 1.0 & 0.055 \\
 &  & Church & 99.9 & \textbf{100.0} & 0.6 & 0.4 & 0.040 \\
\midrule
Eff-Diff & VINE & Bedroom & 83.2 & \textbf{100.0} & 0.3 & 0.0 & 0.041 \\
 &  & Church & 97.7 & \textbf{100.0} & 1.3 & 1.5 & 0.035 \\
\addlinespace[1pt]
 & Trust & Bedroom & 85.1 & \textbf{100.0} & 0.9 & 0.3 & 0.055 \\
 &  & Church & 98.8 & \textbf{100.0} & 1.0 & 0.0 & 0.043 \\
\addlinespace[1pt]
 & Sleeper & Bedroom & 43.9 & \textbf{100.0} & 1.5 & 0.9 & 0.032 \\
 &  & Church & 99.2 & \textbf{100.0} & 0.4 & 0.6 & 0.024 \\
\bottomrule
\end{tabular*}
\end{table}

\begin{figure*}[t]
\centering
\includegraphics[width=0.9\textwidth]{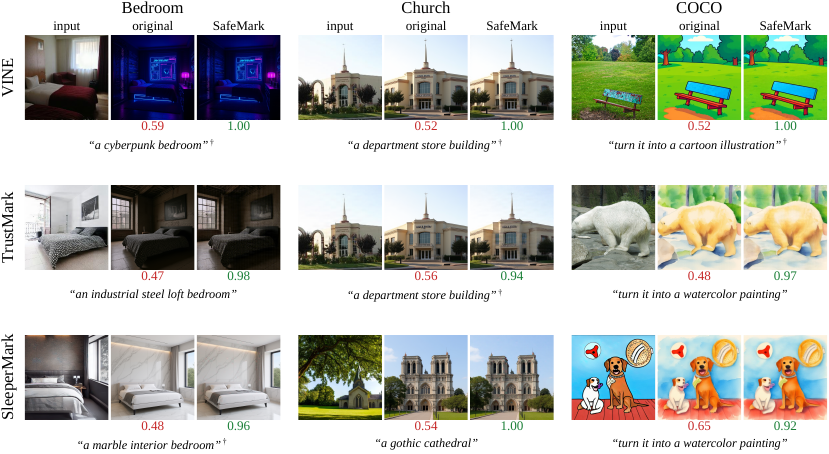}
\caption{\textbf{FLUX.1 Kontext examples} on held-out test images. Each triplet shows the watermarked input, the original edit, and the \textit{SafeMark} edit with its decoded bit accuracy, in green if the watermark is detected at a $1\%$ false-positive rate and in red otherwise. $^\dagger$Held-out instruction.}
\label{fig:kontext_examples}
\end{figure*}

\begin{figure}[t]
\centering
\includegraphics[width=0.9\columnwidth]{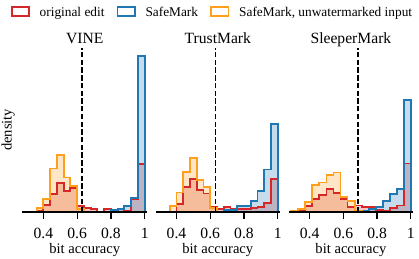}
\caption{\textbf{Per-image bit accuracy} with FLUX.1 Kontext over all test images, datasets, and training seeds. Dashed lines: detection threshold at a $1\%$ false-positive rate.}
\label{fig:bitacc}
\end{figure}

\begin{figure}[t]
\centering
\includegraphics[width=0.9\columnwidth]{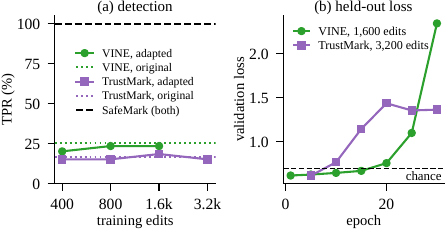}
\caption{\textbf{Adapting the decoder to edited images} (Bedroom, FLUX.1 Kontext). (a) TPR versus the number of training edits. (b) Loss on held-out edits when fine-tuning on the largest training set. The dashed line marks random guessing ($\ln 2$).}
\label{fig:decoder}
\end{figure}

\subsection{Limitations of Existing Robust Watermarking under Text-Guided Editing (RQ1)}\label{sec:eval_rq1}

The Base columns of \cref{tab:kontext,tab:pixel} report detection after direct editing. All three watermarks remain detectable on unedited inputs even after JPEG compression or blurring (gray curves in \cref{fig:attacks}), yet FLUX.1 Kontext edits leave only $16.1$--$23.9\%$ of the Bedroom watermarks and $33.9$--$42.2\%$ of the Church watermarks detectable, and the bit accuracy of these edited images falls to $0.57$--$0.69$. Detection is higher on COCO ($60.6$--$75.0\%$), whose instructions include lighting changes and small local additions, but a quarter to two fifths of the watermarks are still lost. The damage depends on the editor: the same watermarking schemes survive DiffusionCLIP in $95.0$--$100\%$ of cases, whereas Eff-Diff reduces SleeperMark detection on Bedroom to $43.9\%$.
This editor-dependent failure highlights a fundamental limitation of the adversarial-defense paradigm: watermark robustness trained against a fixed distortion distribution cannot generalize to the open-ended space of semantic edits produced by heterogeneous diffusion pipelines.
These observations motivate \textit{SafeMark} that enforces watermark preservation directly within the editing objective.

\begin{figure*}[t]
\centering
\includegraphics[width=0.9\textwidth]{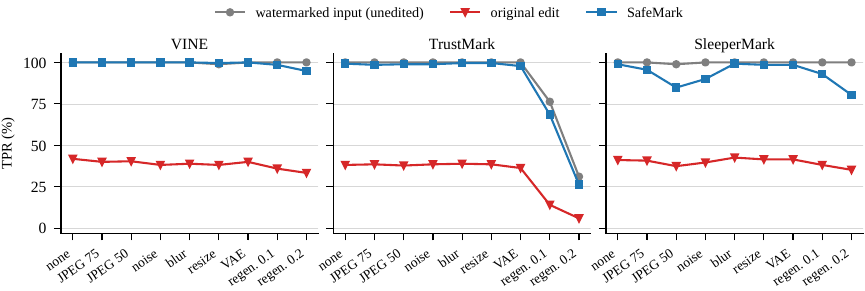}
\caption{\textbf{Detection under post-edit distortions} with FLUX.1 Kontext, averaged over datasets and training seeds. Regen.: diffusion regeneration at strength $0.1$ or $0.2$. \cref{tab:attacks} lists the exact values.}
\label{fig:attacks}
\end{figure*}

\subsection{Watermark-Preserving Editing with \textit{SafeMark} (RQ2)}\label{sec:eval_rq2}

\cref{tab:kontext} reports \textit{SafeMark} with FLUX.1 Kontext. Detection rises to $97.8$--$100\%$ in all nine settings, with a standard deviation of at most $1.9$ points across training seeds, and to $96.7$--$100\%$ on the held-out instructions alone. Bit accuracy rises from $0.57$--$0.84$ to $0.92$--$0.99$. The outputs remain close to the original edits: LPIPS is between $0.019$ and $0.058$ in seven of the nine settings, and the CLIP similarity to the instruction changes by at most $0.007$, so the requested edit is still applied. \cref{fig:teaser,fig:kontext_examples} show that the original edit and the \textit{SafeMark} edit are visually almost indistinguishable, even for instructions that replace the entire scene. The largest deviations occur on COCO with TrustMark (LPIPS $0.166$) and SleeperMark ($0.104$), where \textit{SafeMark} softens fine textures such as brush strokes and rain streaks (\cref{fig:coco_texture}). Raising $\lambda_{\text{sem}}$ from $10$ to $30$ lowers the LPIPS of TrustMark on COCO to $0.062$ at $98.3\%$ detection (\cref{tab:trust_lambda}).
The pixel-space editors in \cref{tab:pixel} show the same pattern. \textit{SafeMark} reaches $100\%$ detection in all twelve settings, including SleeperMark on Bedroom under Eff-Diff, where direct editing retains only $43.9\%$, with LPIPS between $0.024$ and $0.069$ to the original edits (examples in \cref{fig:pixel_examples}).

\noindent\textbf{Ruling out fabrication.}
High detection alone does not show that \textit{SafeMark} preserves the input watermark, because an editor that writes the training message onto every output would detect equally well. Two controls rule this out. First, unwatermarked inputs edited by \textit{SafeMark} are detected in at most $1.1\%$ of cases with FLUX.1 Kontext and $1.5\%$ with the pixel-space editors, comparable to the original editors. \cref{fig:bitacc} shows the underlying distributions: for watermarked inputs, the bit accuracy of \textit{SafeMark} outputs concentrates near $1$, whereas for unwatermarked inputs it stays around $0.5$, below the detection threshold. Second, when an input carries an independently sampled random message, the \textit{SafeMark} output recovers that message with bit accuracy $0.92$--$0.99$ (Own in \cref{tab:kontext}), while its agreement with the training message stays at chance level ($0.50$--$0.51$). The adapted editor therefore preserves the message that the input carries, even when this message never appeared during fine-tuning.
These results answer RQ2 by showing that watermark-safe image editing is achievable at the editing output: \textit{SafeMark} maintains watermark correctness across diverse text-guided image editing pipelines without compromising semantic manipulation.

\subsection{Can the Decoder Recover the Watermark Instead?}\label{sec:eval_decoder}

An alternative to adapting the editor is to keep the editor fixed and adapt the decoder to its edits. We test this with the original FLUX.1 Kontext editor and the VINE and TrustMark decoders. We re-watermark the training images with random messages, edit them with the training instructions, and fine-tune the decoder with a binary cross-entropy loss to recover each embedded message from both the edited and the unedited image. On Bedroom, we enlarge the training set from $400$ to $3{,}200$ edits by drawing several random messages per image and instruction, and hold out $10\%$ of the edits for early stopping.
As \cref{fig:decoder}a shows, the adapted decoder detects the watermark in only $15.0$--$23.3\%$ of the held-out Bedroom edits, within five points of the original decoder, and an eightfold larger training set does not change this. On Church and COCO, the adapted decoder likewise stays within $3.3$ points of the original decoder (\cref{tab:decoder}). \cref{fig:decoder}b explains why: while the training loss keeps decreasing, the loss on held-out edits reaches its minimum within the first five epochs, at $0.58$--$0.64$, barely below the value $0.69$ of random guessing, and then increases, so the decoder memorizes the training edits and does not learn a mapping that transfers to new images. \textit{SafeMark} detects $98.9$--$100\%$ of the watermarks in the same settings. This outcome is consistent with Theorem~\ref{thm:fano}: if direct editing leaves little mutual information between the watermark and the edited image, no decoder can recover the message, and preservation has to take place within the editor.

\subsection{Post-Edit Distortion Robustness (RQ3)}\label{sec:eval_rq3}

We further examine \textit{SafeMark}'s robustness to common post-edit distortions. \cref{fig:attacks} compares detection under eight distortions for the unedited watermarked input, the original FLUX.1 Kontext edit, and the \textit{SafeMark} edit. For VINE and TrustMark, \textit{SafeMark} outputs remain detectable in $97.8$--$100\%$ of cases under JPEG compression, Gaussian noise, blurring, resizing, and VAE reconstruction, within a few points of the unedited inputs. Regeneration at strength $0.2$ is the hardest distortion: VINE retains $94.8\%$, whereas TrustMark retains only $26.3\%$. The unedited TrustMark inputs, however, retain only $31.1\%$, so this failure already occurs without editing and stems from the watermark itself. SleeperMark loses more under \textit{SafeMark} editing, most clearly after JPEG compression at quality $50$, where detection falls from $98.9\%$ for the inputs to $84.8\%$. In every case, \textit{SafeMark} outputs are far more robust than direct edits, of which at most $42.6\%$ remain detectable, and the false-positive rate on distorted unwatermarked outputs stays below $2\%$. \cref{fig:pixel_attacks} reports the corresponding results for the pixel-space editors, including geometric transformations.

\section{Conclusion}\label{sec:conclusion}

We studied watermark preservation under text-guided image editing and introduced \textit{SafeMark}, which enforces watermark consistency directly within the editing objective. Across pixel-space diffusion editors and FLUX.1 Kontext, \textit{SafeMark} keeps embedded watermarks detectable while staying close to the original edits, whereas adapting the decoder to edited images does not recover the lost watermark. More broadly, our results show that watermark degradation stems from unconstrained editing dynamics and can be avoided without sacrificing semantic editing. Because an edit can remove the watermark information beyond the reach of any decoder, provenance has to be protected during editing, by the editor itself. This reframes provenance protection as an edit-time design problem and motivates co-designing watermarking systems with editing pipelines and tamper-aware objectives. It also gives editing services a practical way to respect the ownership signals of the images they receive, since the adapted editor needs neither the watermark encoder nor the message.

{
    \small
    \bibliographystyle{ieeenat_fullname}
    \bibliography{references}
}

\clearpage
\maketitlesupplementary
\appendix
\setcounter{table}{0}\setcounter{figure}{0}\setcounter{theorem}{0}\setcounter{lemma}{0}\setcounter{equation}{0}
\renewcommand{\thetable}{S\arabic{table}}\def\theHtable{S\arabic{table}}
\renewcommand{\thefigure}{S\arabic{figure}}\def\theHfigure{S\arabic{figure}}
\renewcommand{\thetheorem}{S\arabic{theorem}}\def\theHtheorem{S\arabic{theorem}}
\renewcommand{\thelemma}{S\arabic{lemma}}\def\theHlemma{S\arabic{lemma}}
\renewcommand{\theequation}{S\arabic{equation}}\def\theHequation{S\arabic{equation}}
This supplementary material contains the proofs of our theoretical results (\cref{app:theory}), experimental details (\cref{app:details}), additional FLUX.1 Kontext results (\cref{app:kontext}), results for the pixel-space and Stable Diffusion editors (\cref{app:other_editors}), and a comparison with re-embedding the decoded message after editing (\cref{app:reembed}).

\section{Proofs of Theoretical Results}\label{app:theory}

\subsection*{Proof of Theorem~\ref{thm:fano}}

Fano's inequality applied to $\hat W$ as an estimate of $W$ gives
$H(W\mid\hat W)\le H_2(P_e)+P_e\log(|\mathcal{W}|-1)\le 1+P_e\,H(W)$,
since $H_2(P_e)\le 1$ bit and (for uniform $P_W$) $\log|\mathcal{W}|=H(W)$.
By the data-processing inequality (DPI) on $W\to X_e\to\hat W$,
$H(W\mid\hat W)=H(W)-I(W;\hat W)\ge H(W)-I_\theta(W;X_e)$.
Combining and rearranging yields Eq.~\eqref{eq:fano}.\hfill$\square$

\paragraph{Remark.}
Theorem~\ref{thm:fano} is only a \emph{necessary} condition and does not certify low $P_e$.
Average bit accuracy and block recovery are distinct quantities. Lemma~\ref{lem:acc-mi} supplies a lower bound on information, whereas Theorem~\ref{thm:fano} gives an error lower bound from the actual information or from an upper bound on it. Substituting an information lower bound into the right-hand side of Eq.~\eqref{eq:fano} does not yield a valid error lower bound. Under independent bit errors with correctness probability $a$, exact recovery has probability $a^B$, which illustrates why detection, average bit accuracy, and complete message recovery must be reported separately.
We therefore use Eq.~\eqref{eq:capacity-from-tau} only as a principled formal motivation for bit-level training. It is not a sufficient condition for exact $W$-recovery.

\subsection*{Proof of Lemma~\ref{lem:acc-mi}}

For each bit $b$, applying Fano's inequality to the binary channel $W_b\to\hat W_b$
($|\mathcal{W}_b|=2$, so $\log(|\mathcal{W}_b|-1)=0$) gives $H(W_b\mid\hat W_b)\le H_2(1-a_b)$.
Expanding $H(W\mid\hat W)$ by the chain rule and using that conditioning on a richer $\sigma$-algebra reduces entropy,
\begin{align*}
H(W\mid\hat W)
&\;=\;\sum_{b=1}^B H(W_b\mid W_{<b},\hat W)\\
&\;\le\;\sum_{b=1}^B H(W_b\mid\hat W_b)
\;\le\;\sum_{b=1}^B H_2(1-a_b),
\end{align*}
where the first inequality holds because $(W_{<b},\hat W)$ contains $\hat W_b$ as a sub-component.
By concavity of $H_2$ on $[0,1]$ (Jensen), $\sum_b H_2(1-a_b)\le B\,H_2(1-\bar a)$.
Finally, the DPI on $W\to X_e\to\hat W$ gives
$I_\theta(W;X_e)\ge I(W;\hat W)=H(W)-H(W\mid\hat W)\ge H(W)-B\,H_2(1-\bar a)$.\hfill$\square$

\begin{lemma}[Soft-to-hard calibration]\label{lem:calibration}
$\widetilde{\mathrm{Acc}}(\hat{\mathbf{w}},\mathbf{w})\ge\tau$ implies hard-bit accuracy $\bar a\ge 4\tau-3$.
\end{lemma}

\subsection*{Proof of Lemma~\ref{lem:calibration}}

Let $e:=B(1-\bar a^{\mathrm{s}})$ be the number of hard-wrong bits in the sample.
For any wrong bit, $|\hat w_b-w_b|\ge 1/2$, hence $1-(\hat w_b-w_b)^2\le 3/4$.
For any correct bit, $1-(\hat w_b-w_b)^2\le 1$. Therefore
\[
\widetilde{\mathrm{Acc}}^{\mathrm{s}}
\;\le\;
\frac{(B-e)\cdot 1+e\cdot(3/4)}{B}
\;=\;1-\frac{e}{4B}
\;=\;1-\frac{1-\bar a^{\mathrm{s}}}{4},
\]
which rearranges to $\bar a^{\mathrm{s}}\ge 4\widetilde{\mathrm{Acc}}^{\mathrm{s}}-3$.
Taking expectation preserves the linear inequality.\hfill$\square$

\paragraph{Remark.}
For $\tau\in[1/2,7/8)$ the calibration bound $\bar a\ge 4\tau-3$ becomes vacuous (it implies $\bar a<1/2$, where Lemma~\ref{lem:acc-mi} is trivial).
Tightening the calibration in this regime via cross-entropy surrogates, whose direct Barber--Agakov bound gives $I_\theta\ge H(W)-\mathbb{E}[\mathrm{BCE}]$, is straightforward but unused in our experiments.
SafeMark's reported configuration uses $\tau=1$ throughout, where zero hinge loss implies perfect hard-bit recovery. With sigmoid outputs, however, $\widetilde{\mathrm{Acc}}<1$, so the hinge does not reach zero at this threshold.

\subsection*{Theorem~\ref{thm:recovery}: Finite-Step Entry into the Inactive Phase}

\begin{theorem}[Finite-step entry into the inactive phase]
\label{thm:recovery}
Assume:
\begin{itemize}
\setlength\itemsep{0pt}
  \item[\textnormal{(A1)}] $\widetilde{\mathrm{Acc}}(\cdot)$ is $L$-smooth in $\theta$.
  \item[\textnormal{(A2)}] In the active phase, $\|\nabla_\theta\widetilde{\mathrm{Acc}}(\theta)\|\ge\mu>0$.
  \item[\textnormal{(A3)}] In the active phase, $\langle\nabla\mathcal{L}_{\mathrm{wm}}(\theta),\nabla\mathcal{L}_{\mathrm{total}}(\theta)\rangle\ge c\|\nabla\mathcal{L}_{\mathrm{wm}}(\theta)\|^2$ for some $c>0$.
  \item[\textnormal{(A4)}] $\|\nabla_\theta\mathcal{L}_{\mathrm{total}}(\theta)\|\le M$ for all iterates $\theta$.
\end{itemize}
Then gradient descent $\theta_{t+1}=\theta_t-\eta\nabla\mathcal{L}_{\mathrm{total}}(\theta_t)$ with step size $\eta\le c\mu^2/(LM^2)$, starting from any $\theta_0$ in the active phase, enters the inactive phase $\{\widetilde{\mathrm{Acc}}\ge\tau\}$ within
\begin{equation}
\label{eq:recovery-bound}
T \;\le\; \frac{2\,\mathcal{L}_{\mathrm{wm}}(\theta_0)}{\eta c\mu^{2}}
\end{equation}
iterations.
\end{theorem}

\begin{proof}
Let $f(\theta):=\tau-\widetilde{\mathrm{Acc}}(\theta)$, so $\mathcal{L}_{\mathrm{wm}}(\theta)=\max(0,f(\theta))$.
By (A1), $f$ is $L$-smooth.
For any $\theta_t$ in the active phase ($f(\theta_t)=\mathcal{L}_{\mathrm{wm}}(\theta_t)>0$), the smoothness inequality gives
\begin{align*}
f(\theta_{t+1})
\;\le\;
f(\theta_t)
&-\eta\langle\nabla\mathcal{L}_{\mathrm{wm}}(\theta_t),\,\nabla\mathcal{L}_{\mathrm{total}}(\theta_t)\rangle\\
&+\frac{L\eta^{2}}{2}\|\nabla\mathcal{L}_{\mathrm{total}}(\theta_t)\|^{2}.
\end{align*}
Applying (A3) and $\|\nabla\mathcal{L}_{\mathrm{wm}}\|\ge\mu$ from (A2), the inner-product term is at least $c\mu^2$.
Applying (A4), the second-order term is at most $L\eta^2 M^2/2$. Hence
\[
f(\theta_{t+1})\;\le\;\mathcal{L}_{\mathrm{wm}}(\theta_t)-\eta c\mu^{2}+\frac{L\eta^{2}M^{2}}{2}.
\]
The choice $\eta\le c\mu^{2}/(LM^{2})$ gives $L\eta^2M^2/2\le \eta c\mu^2/2$, so
\[
f(\theta_{t+1})\;\le\;\mathcal{L}_{\mathrm{wm}}(\theta_t)-\frac{\eta c\mu^{2}}{2}.
\]
Since $\mathcal{L}_{\mathrm{wm}}\ge 0$, this descent of at least $\eta c\mu^2/2$ per active-phase step cannot continue beyond $T=2\mathcal{L}_{\mathrm{wm}}(\theta_0)/(\eta c\mu^2)$ steps without entering the inactive phase.
\end{proof}

\paragraph{Discussion of assumptions.}
(A1) follows from the Brier surrogate definition composed with the editor and decoder and is the mildest of the four.
(A4) can be enforced by gradient clipping or by working on a bounded parameter ball.
(A3) holds whenever $\lambda_{\mathrm{wm}}$ is chosen sufficiently large relative to $\lambda_{\mathrm{sem}}G_{\mathrm{sem}}/\mu$ (under a uniform bound $\|\nabla\mathcal{L}_{\mathrm{sem}}\|\le G_{\mathrm{sem}}$).
(A2) is the strongest assumption: a P\L-type lower bound on $\|\nabla\widetilde{\mathrm{Acc}}\|$ uniformly on the active phase is not provable for general non-convex deep networks.
We treat it as a working assumption consistent with the rapid empirical convergence observed in Section~\ref{sec:eva}, and view Theorem~\ref{thm:recovery} only as a \emph{conditional} convergence result. Because the inactive phase is reachable only for $\tau<1$, the theorem does not describe the $\tau=1$ configuration used in our experiments, which are trained with Adam.
Theorem~\ref{thm:recovery} extends to mini-batch SGD by replacing the deterministic budget with an expected-hitting-time analogue under a variance-bound assumption on the stochastic gradient.

\section{Experimental Details}\label{app:details}

\noindent\textbf{Choice of watermarks.}
We select VINE, TrustMark, and SleeperMark as the most directly comparable watermarks. VINE and TrustMark are encoder--decoder schemes for real images, and SleeperMark embeds its watermark while generating images. Robust-Wide~\cite{hu2024robust} follows the same watermark-side strategy as VINE with InstructPix2Pix as its editing layer, JigMark~\cite{pan2024jigmark} relies on a black-box jigsaw-key paradigm that is incompatible with encoder--decoder schemes, and EditGuard~\cite{zhang2024editguard} addresses tamper localization, a task different from bit-level message recovery.

\subsection{FLUX.1 Kontext}\label{app:kontext_details}

\noindent\textbf{Images and watermarks.}
LSUN images are upsampled from $256$ to $512$ pixels before they are encoded by the editor, and all edits are resized back to the dataset resolution ($256$ pixels for LSUN and $512$ pixels for COCO) before evaluation. The VINE and TrustMark decoders read the edit at $256$ pixels and the SleeperMark decoder at $512$ pixels. COCO images are $100$ photos from COCO val2017, center-cropped and resized to $512\times512$. For COCO, TrustMark embeds at $512$ pixels, while VINE embeds at $256$ pixels and its residual is upsampled to $512$ pixels and added to the photo. SleeperMark COCO images are generated by the SleeperMark model from COCO captions with its trigger. TrustMark uses its Q model with $100$ raw bits and no error correction.

\noindent\textbf{Instructions.}
\cref{tab:instructions} lists the instructions. For LSUN, an instruction has the form \emph{turn it into} followed by the listed target, and the COCO instructions are used verbatim. Each held-out test image is edited with one training instruction and one held-out instruction, assigned cyclically.

\noindent\textbf{Training settings.}
We use $28$ sampling steps with guidance scale $2.5$ and adapt the last $6$ steps (\cref{sec:large}). We train with Adam at learning rate $10^{-3}$ for $400$ steps for VINE and $1{,}000$ steps for TrustMark and SleeperMark, with $\tau=1$ and $\lambda_{\text{wm}}=3$, which is increased linearly over the first $20\%$ of the steps. We set $\lambda_{\text{sem}}=10$ for VINE and TrustMark. For SleeperMark, we use the larger value in $\{30,100\}$ for which the training bit accuracy still reaches $0.95$. This rule uses training images only and selects $100$ for Bedroom and COCO and $30$ for Church (single runs in \cref{tab:design}). The three seeds of every configuration use the selected values.

\noindent\textbf{Fine-tuning.}
The LoRA adapters have rank $16$ and $\alpha=16$, are initialized with Gaussian down-projections and zero up-projections, and are kept in fp32, while the transformer runs in bf16. Each training step samples one training image, one training instruction, and one noise seed. For TrustMark and SleeperMark, the shared early steps and the reference edits are precomputed for all $400$ pairs of the $50$ training images and $8$ training instructions with one noise seed each, and a step takes about $12$\,s on one GPU. VINE trains for $400$ steps without precomputation.

\noindent\textbf{Decoder adaptation.}
The baseline in \cref{sec:eval_decoder} uses the $50$ training images. Each image is re-watermarked with a random message for each of its $8$ training instructions (and, for the Bedroom learning curves, with $4$ messages for VINE and $8$ for TrustMark per instruction), and edited by the original FLUX.1 Kontext editor. A copy of the decoder is fine-tuned with Adam at learning rate $10^{-4}$ and batch size $16$ on the sum of the binary cross-entropy losses for the edited image and for the unedited re-watermarked image, the latter keeping the decoder valid on unedited inputs. The $400$-edit runs train for $20$ epochs. The learning-curve runs train for $30$ epochs on nested subsets of one shuffled pool, keep $10\%$ of each subset for validation, and use the epoch with the lowest validation loss. All adapted decoders still decode the unedited watermarked test images with $97.9$--$99.9\%$ bit accuracy.

\noindent\textbf{Post-edit distortions.}
JPEG compression uses quality $50$ or $75$. Gaussian noise has standard deviation $0.05$ for pixel values in $[0,1]$. Blur is Gaussian with $\sigma=1$. Resizing scales the image to half its size and back. VAE reconstruction encodes and decodes the image with the Stable Diffusion v1.5 autoencoder at $512$ pixels, and regeneration runs Stable Diffusion v1.5 image-to-image with an empty prompt, strength $0.1$ or $0.2$, and $50$ DDIM steps. Distortions are applied to one held-out-instruction edit per test image ($30$ images) and to the corresponding edits of the $30$ unwatermarked images.

\begin{table}[t]
\centering
\caption{\textbf{Editing instructions for FLUX.1 Kontext.} LSUN instructions read \emph{turn it into} followed by the target.}
\label{tab:instructions}
\scriptsize
\setlength{\tabcolsep}{3pt}
\begin{tabular}{p{0.11\columnwidth}p{0.45\columnwidth}p{0.36\columnwidth}}
\toprule
Data & Training (8) & Held out (4) \\
\midrule
Bedroom & a bamboo room, a golden luxurious bedroom, an industrial steel loft bedroom, a royal palace bedroom, a bedroom overgrown by a forest, a bedroom in the desert, a vintage bedroom, a japanese tatami room & an underwater bedroom, a cyberpunk bedroom, a marble interior bedroom, a seaside bedroom \\
\midrule
Church & a modern skyscraper, a gothic cathedral, a church in an autumn scene, a factory building, an ancient ruin, a red brick building, an iron structure building, a medieval castle & a department store building, a pagoda, a temple, a baroque church \\
\midrule
COCO & make it a snowy winter scene, turn it into a watercolor painting, make it look like it is at night, turn it into a pencil sketch, make it look like an old vintage photograph, change the lighting to a warm sunset glow, add a small red balloon floating in the air, make the main object blue & turn it into an oil painting in the style of Van Gogh, make it look like it is raining, turn it into a cartoon illustration, add a butterfly to the scene \\
\bottomrule
\end{tabular}
\end{table}

\subsection{Pixel-Space Editors}\label{app:pixel_details}

For each of DiffusionCLIP and Eff-Diff, we train a separate model for every editing attribute and watermarking scheme in each scene. DiffusionCLIP uses inversion timestep $500$ and Eff-Diff uses timestep $350$. Losses apply to the final edited image, and gradients propagate through the generation steps, six by default. TrustMark with DiffusionCLIP uses twenty training steps and forty evaluation steps, gradient checkpointing, and learning rate $4\times10^{-6}$. The other DiffusionCLIP runs use $8\times10^{-6}$, and Eff-Diff uses $6\times10^{-6}$. Adam~\cite{kingma2014adam} runs for five epochs with a learning-rate multiplier of $1.3$ per epoch. The final configurations use $\lambda_{\text{wm}}=1$, which increases linearly over these epochs, and $\tau=1$. The anchor weight $\lambda_{\text{sem}}$ is $10$ for SleeperMark, $3$ for VINE, and $3$ or $10$ for TrustMark with Eff-Diff or DiffusionCLIP, respectively. Each pixel-space result averages the available attributes with equal weight, with one trained model per attribute.

\subsection{Stable Diffusion Editors}\label{app:sd_details}

For DDIM inversion and InstructPix2Pix on Stable Diffusion v1.5, \textit{SafeMark} fine-tunes the attention layers of the denoising U-Net for $400$ steps at learning rate $5\times10^{-6}$, with $\lambda_{\text{sem}}=10$, $\lambda_{\text{wm}}=3$ warmed up over the first $20\%$ of the steps, $20$ sampling steps, and gradients through the final $4$ steps and the VAE decoder. DDIM inversion uses inversion depth $0.6$ and text guidance $7$, and InstructPix2Pix uses text guidance $7.5$ and image guidance $1.2$. Images, instruction splits, and the three training seeds follow the FLUX.1 Kontext protocol.

\subsection{Detection Statistics}\label{app:stats}

The detection threshold is the smallest integer $k$ with $\Pr[\operatorname{Binom}(B,1/2)\ge k]\le0.01$, which gives $63$ for $B=100$ and $33$ for $B=48$. Because the independent fair-bit null is only an assumption for a particular decoder, we report the measured FPR on unwatermarked images alongside every TPR.

\section{Additional FLUX.1 Kontext Results}\label{app:kontext}

\noindent\textbf{Training and held-out instructions.}
\cref{tab:split} separates the results of \cref{tab:kontext} by instruction split. Detection on held-out instructions ($96.7$--$100\%$) is close to that on training instructions ($98.9$--$100\%$), so the adapted editor does not depend on the instructions seen during fine-tuning. On COCO, LPIPS is higher for the held-out instructions, which include the texture-rich Van Gogh and rain edits.

\begin{table}[t]
\centering
\caption{\textbf{FLUX.1 Kontext results by instruction split}, averaged over three training seeds.}
\label{tab:split}
\scriptsize
\setlength{\tabcolsep}{2.2pt}
\begin{tabular*}{\columnwidth}{@{\extracolsep{\fill}}llcccccccc@{}}
\toprule
\multirow{3}{*}{WM} & \multirow{3}{*}{Data} & \multicolumn{4}{c}{Training instructions} & \multicolumn{4}{c}{Held-out instructions} \\
\cmidrule(lr){3-6} \cmidrule(lr){7-10}
 & & \multicolumn{2}{c}{TPR (\%)} & \multirow{2}{*}{Acc} & \multirow{2}{*}{LPIPS} & \multicolumn{2}{c}{TPR (\%)} & \multirow{2}{*}{Acc} & \multirow{2}{*}{LPIPS} \\
\cmidrule(lr){3-4} \cmidrule(lr){7-8}
 & & Base & SM & & & Base & SM & & \\
\midrule
VINE & Bedroom & 23.3 & 100.0 & 0.990 & 0.024 & 24.4 & 100.0 & 0.991 & 0.018 \\
 & Church & 40.0 & 100.0 & 0.972 & 0.022 & 44.4 & 100.0 & 0.973 & 0.018 \\
 & COCO & 72.2 & 100.0 & 0.990 & 0.040 & 57.8 & 100.0 & 0.987 & 0.076 \\
\addlinespace[2pt]
TrustMark & Bedroom & 15.6 & 100.0 & 0.967 & 0.036 & 16.7 & 100.0 & 0.974 & 0.028 \\
 & Church & 41.1 & 98.9 & 0.929 & 0.059 & 42.2 & 100.0 & 0.914 & 0.050 \\
 & COCO & 64.4 & 100.0 & 0.937 & 0.123 & 56.7 & 97.8 & 0.924 & 0.209 \\
\addlinespace[2pt]
SleeperMark & Bedroom & 15.6 & 98.9 & 0.956 & 0.022 & 20.0 & 100.0 & 0.949 & 0.016 \\
 & Church & 36.7 & 100.0 & 0.953 & 0.044 & 31.1 & 100.0 & 0.942 & 0.039 \\
 & COCO & 75.6 & 98.9 & 0.960 & 0.067 & 74.4 & 96.7 & 0.918 & 0.141 \\
\bottomrule
\end{tabular*}
\end{table}

\noindent\textbf{Exact values for the figures.}
\cref{tab:decoder} lists the TPRs plotted in \cref{fig:decoder}a together with the $400$-edit runs on Church and COCO, and \cref{tab:attacks} lists the TPRs plotted in \cref{fig:attacks}.

\begin{table}[t]
\centering
\caption{\textbf{Decoder adaptation} (TPR, \%) on held-out FLUX.1 Kontext edits: original decoder (Orig.), decoders fine-tuned on $400$--$3{,}200$ edits, and \textit{SafeMark}.}
\label{tab:decoder}
\footnotesize
\setlength{\tabcolsep}{3.4pt}
\begin{tabular*}{\columnwidth}{@{\extracolsep{\fill}}llcccccc@{}}
\toprule
\multirow{2}{*}{WM} & \multirow{2}{*}{Data} & \multirow{2}{*}{Orig.} & \multicolumn{4}{c}{Adapted decoder, edits} & \multirow{2}{*}{\textit{SafeMark}} \\
\cmidrule(lr){4-7}
 & & & 400 & 800 & 1600 & 3200 & \\
\midrule
VINE & Bedroom & 25.0 & 20.0 & 23.3 & 23.3 & -- & \textbf{100.0} \\
 & Church & 41.7 & 41.7 & -- & -- & -- & \textbf{100.0} \\
 & COCO & 63.3 & 60.0 & -- & -- & -- & \textbf{100.0} \\
\addlinespace[2pt]
TrustMark & Bedroom & 16.7 & 15.0 & 15.0 & 18.3 & 15.0 & \textbf{100.0} \\
 & Church & 41.7 & 43.3 & -- & -- & -- & \textbf{99.4} \\
 & COCO & 61.7 & 63.3 & -- & -- & -- & \textbf{98.9} \\
\bottomrule
\end{tabular*}
\end{table}

\begin{table}[t]
\centering
\caption{\textbf{Detection under post-edit distortions} (TPR, \%) with FLUX.1 Kontext, averaged over datasets and training seeds. Input, Base, and SM denote the unedited watermarked image, the original edit, and the \textit{SafeMark} edit.}
\label{tab:attacks}
\footnotesize
\setlength{\tabcolsep}{1.9pt}
\begin{tabular*}{\columnwidth}{@{\extracolsep{\fill}}llcccccccc@{}}
\toprule
WM & Image & \makecell{JPEG\\50} & \makecell{JPEG\\75} & Noise & Blur & \makecell{Re-\\size} & VAE & \makecell{Regen.\\0.1} & \makecell{Regen.\\0.2} \\
\midrule
VINE & Input & 100.0 & 100.0 & 100.0 & 100.0 & 98.9 & 100.0 & 100.0 & 100.0 \\
 & Base & 40.4 & 40.0 & 38.1 & 38.9 & 38.1 & 40.0 & 35.9 & 33.3 \\
 & \textbf{SM} & 100.0 & 100.0 & 100.0 & 100.0 & 99.6 & 100.0 & 98.5 & 94.8 \\
\addlinespace[2pt]
Trust & Input & 100.0 & 100.0 & 100.0 & 100.0 & 100.0 & 100.0 & 76.3 & 31.1 \\
 & Base & 37.8 & 38.5 & 38.5 & 38.9 & 38.5 & 36.3 & 14.1 & 5.9 \\
 & \textbf{SM} & 98.9 & 98.5 & 98.9 & 99.6 & 99.6 & 97.8 & 68.5 & 26.3 \\
\addlinespace[2pt]
Sleeper & Input & 98.9 & 100.0 & 100.0 & 100.0 & 100.0 & 100.0 & 100.0 & 100.0 \\
 & Base & 37.4 & 40.7 & 39.6 & 42.6 & 41.5 & 41.5 & 38.1 & 35.2 \\
 & \textbf{SM} & 84.8 & 95.6 & 90.0 & 99.3 & 98.5 & 98.5 & 93.0 & 80.4 \\
\bottomrule
\end{tabular*}
\end{table}

\noindent\textbf{Fine textures on COCO.}
TrustMark and SleeperMark on COCO show the largest deviation from the original edits in \cref{tab:kontext}. \cref{fig:coco_texture} shows that the deviation is concentrated in fine, high-frequency textures that the instruction asks for, such as brush strokes and rain streaks, which \textit{SafeMark} renders more softly, while the global content of the edit is unchanged.

\begin{figure}[t]
\centering
\includegraphics[width=\columnwidth]{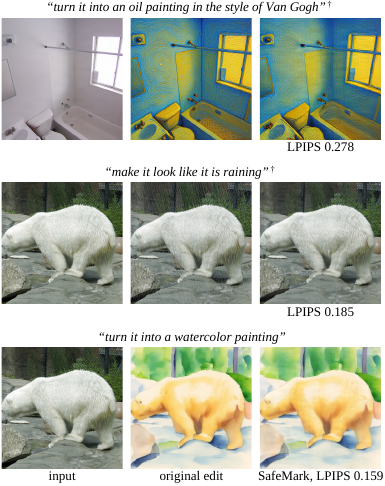}
\caption{\textbf{COCO TrustMark examples with the largest deviations} from the original edits. $^\dagger$Held-out instruction.}
\label{fig:coco_texture}
\end{figure}

\noindent\textbf{Design comparisons.}
\cref{tab:design} lists single runs that compare the adapted sampling steps, the semantic anchor $\lambda_{\text{sem}}$, and the training length.
In a preliminary configuration on Bedroom with TrustMark, the adapters were active in all $28$ sampling steps (trained for $400$ steps at learning rate $10^{-4}$). It also reached full detection, but its outputs drifted from the original edits (LPIPS $0.331$, PSNR $19.9$\,dB) and matched the instruction less well (CLIP-T $0.284$ \vs $0.309$). Some edits, such as \emph{a bamboo room}, were no longer applied. The final configuration with late-step adaptation (\cref{sec:large}) reaches the same detection at LPIPS $0.050$.
A larger $\lambda_{\text{sem}}$ keeps the outputs closer to the original edits but makes the watermark harder to learn. For SleeperMark on Bedroom, raising $\lambda_{\text{sem}}$ from $30$ to $100$ lowers LPIPS from $0.076$ to $0.018$ at unchanged detection, at the cost of JPEG robustness. On Church, the same change lowers detection from $100\%$ to $75\%$ after $1{,}000$ steps, and training for $2{,}000$ steps restores $100\%$ detection at LPIPS $0.020$. The rule in \cref{app:kontext_details} therefore selects $\lambda_{\text{sem}}=30$ for Church at the standard training length.

\noindent\textbf{Semantic anchor for TrustMark.}
\cref{tab:trust_lambda} varies $\lambda_{\text{sem}}$ for TrustMark. On every dataset, a larger anchor brings the outputs closer to the original edits and lowers the final training bit accuracy. On COCO, $\lambda_{\text{sem}}=30$ lowers LPIPS from $0.166$ to $0.062$ while detection stays at $98.3\%$ over three seeds, and $\lambda_{\text{sem}}=100$ lowers LPIPS to $0.039$ but reduces detection to $84.4\%$. Training for $2{,}000$ iterations at $\lambda_{\text{sem}}=100$ raises detection to $93.9\%$ at LPIPS $0.044$ over three seeds, so part of this loss comes from the training length. On Bedroom, $\lambda_{\text{sem}}=30$ keeps full detection in all three seeds and halves LPIPS from $0.032$ to $0.016$. Church is the most sensitive dataset, where detection falls to $93.3\%$ at $\lambda_{\text{sem}}=30$ and to $36.7\%$ at $\lambda_{\text{sem}}=100$. \cref{tab:kontext} uses $\lambda_{\text{sem}}=10$ for TrustMark on all datasets. Where fidelity matters more than the last points of detection, a larger $\lambda_{\text{sem}}$ can be used.

\begin{table*}[t]
\centering
\caption{\textbf{Single-run design comparisons for FLUX.1 Kontext.} Train Acc: final training bit accuracy. JPEG 50: TPR after JPEG compression at quality $50$.}
\label{tab:design}
\footnotesize
\setlength{\tabcolsep}{4pt}
\begin{tabular*}{\textwidth}{@{\extracolsep{\fill}}lllcccccc@{}}
\toprule
WM & Data & Setting & Train Acc & TPR (\%) & LPIPS & PSNR & CLIP-T (SM / Base) & JPEG 50 \\
\midrule
TrustMark & Bedroom & all 28 steps, 400 it., lr $10^{-4}$, $\lambda_{\text{sem}}{=}10$ & 0.894 & 100.0 & 0.331 & 19.9 & 0.284 / 0.309 & 100.0 \\
 &  & last 6 steps, 1{,}000 it., $\lambda_{\text{sem}}{=}10$ & 0.984 & 100.0 & 0.050 & 33.7 & 0.307 / 0.309 & 100.0 \\
\addlinespace[2pt]
SleeperMark & Bedroom & $\lambda_{\text{sem}}{=}30$, 1{,}000 it. & 0.983 & 100.0 & 0.076 & 33.9 & 0.305 / 0.316 & 100.0 \\
 &  & $\lambda_{\text{sem}}{=}100$, 1{,}000 it. & 0.972 & 100.0 & 0.018 & 39.4 & 0.315 / 0.316 & 80.0 \\
\addlinespace[2pt]
SleeperMark & Church & $\lambda_{\text{sem}}{=}30$, 1{,}000 it. & 0.983 & 100.0 & 0.034 & 32.6 & 0.251 / 0.250 & 83.3 \\
 &  & $\lambda_{\text{sem}}{=}100$, 1{,}000 it. & 0.855 & 75.0 & 0.022 & 34.0 & 0.252 / 0.250 & 40.0 \\
 &  & $\lambda_{\text{sem}}{=}100$, 2{,}000 it. & 0.973 & 100.0 & 0.020 & 34.3 & 0.253 / 0.252 & 93.3 \\
\bottomrule
\end{tabular*}
\end{table*}

\begin{table}[t]
\centering
\caption{\textbf{Semantic anchor $\lambda_{\text{sem}}$ for TrustMark} with FLUX.1 Kontext. Rows with $\pm$ average three training seeds ($\pm$ standard deviation), and the other rows are single runs. It.: training iterations. Train Acc: final training bit accuracy.}
\label{tab:trust_lambda}
\scriptsize
\setlength{\tabcolsep}{0pt}
\begin{tabular*}{\columnwidth}{@{\extracolsep{\fill}}lcccccccc@{}}
\toprule
Data & $\lambda_{\text{sem}}$ & It. & \makecell{Train\\Acc} & TPR (\%) & FPR (\%) & Acc & LPIPS & PSNR \\
\midrule
COCO & 10 & 1{,}000 & 0.975 & 98.9{\scriptsize$\pm$1.9} & 1.1 & 0.930 & 0.166{\scriptsize$\pm$0.045} & 27.7 \\
 & 30 & 1{,}000 & 0.964 & 98.3{\scriptsize$\pm$2.9} & 1.7 & 0.923 & 0.062{\scriptsize$\pm$0.014} & 31.1 \\
 & 100 & 1{,}000 & 0.925 & 84.4{\scriptsize$\pm$2.5} & 0.0 & 0.841 & 0.039{\scriptsize$\pm$0.010} & 32.6 \\
 & 100 & 2{,}000 & 0.941 & 93.9{\scriptsize$\pm$1.0} & 0.0 & 0.890 & 0.044{\scriptsize$\pm$0.004} & 31.9 \\
\addlinespace[2pt]
Bedroom & 10 & 1{,}000 & 0.989 & 100.0{\scriptsize$\pm$0.0} & 1.1 & 0.970 & 0.032{\scriptsize$\pm$0.016} & 35.8 \\
 & 30 & 1{,}000 & 0.973 & 100.0{\scriptsize$\pm$0.0} & 0.6 & 0.952 & 0.016{\scriptsize$\pm$0.002} & 38.5 \\
 & 100 & 1{,}000 & 0.811 & 96.7 & 0.0 & 0.836 & 0.013 & 39.5 \\
\addlinespace[2pt]
Church & 10 & 1{,}000 & 0.961 & 99.4{\scriptsize$\pm$1.0} & 1.1 & 0.921 & 0.054{\scriptsize$\pm$0.001} & 30.5 \\
 & 30 & 1{,}000 & 0.885 & 93.3 & 5.0 & 0.838 & 0.027 & 33.4 \\
 & 100 & 1{,}000 & 0.684 & 36.7 & 1.7 & 0.660 & 0.009 & 37.8 \\
\bottomrule
\end{tabular*}
\end{table}

\section{Pixel-Space and Stable Diffusion Editors}\label{app:other_editors}

\begin{figure*}[t]
\centering
\includegraphics[width=0.9\textwidth]{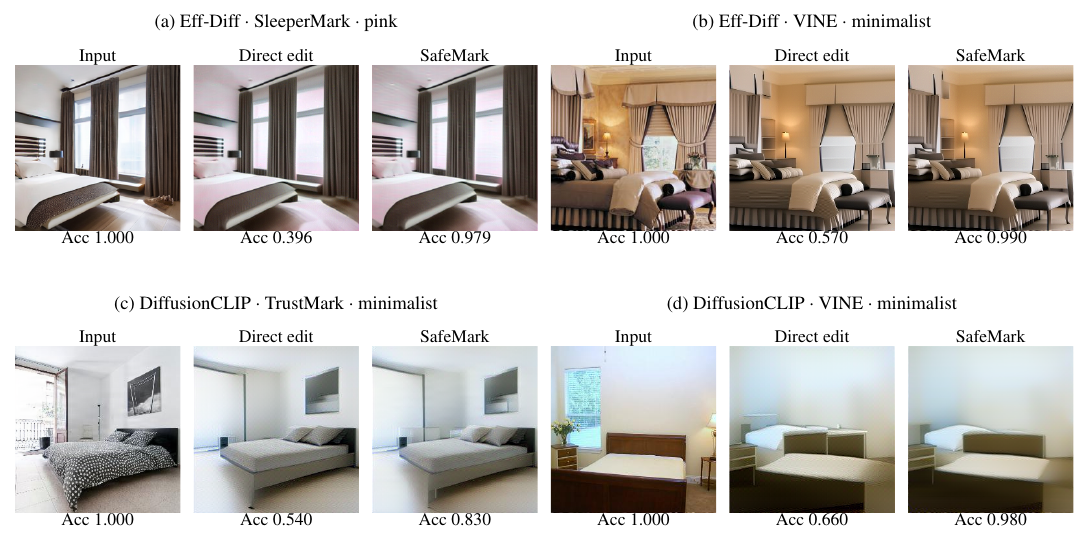}
\caption{\textbf{Pixel-space editor examples} on held-out Bedroom images. Each triplet shows the watermarked input, the original edit, and the \textit{SafeMark} output.}
\label{fig:pixel_examples}
\end{figure*}

\noindent\textbf{Random-message inputs.}
\cref{tab:random_pixel} repeats the random-message control of \cref{sec:eval_rq2} for the pixel-space editors. VINE recovers the embedded message at $0.995$--$0.997$ bit accuracy and TrustMark at $0.953$--$0.987$, while agreement with the training message stays at about $0.52$.

\begin{table}[t]
\centering
\caption{\textbf{Random-message control for the pixel-space editors.} Own and Key: bit accuracy of \textit{SafeMark} outputs with respect to the embedded random message and the training message.}
\label{tab:random_pixel}
\small
\setlength{\tabcolsep}{4pt}
\begin{tabular*}{\columnwidth}{@{\extracolsep{\fill}}llcccc@{}}
\toprule
\multirow{2}{*}{Editor} & \multirow{2}{*}{WM} & \multirow{2}{*}{Own} & \multirow{2}{*}{Key} & \multicolumn{2}{c}{FPR (\%)} \\
\cmidrule(lr){5-6}
 & & & & Base & SM \\
\midrule
DiffusionCLIP & TrustMark & 0.953 & 0.517 & 0.7 & 0.5 \\
 & VINE & 0.995 & 0.517 & 0.2 & 0.6 \\
\midrule
Eff-Diff & TrustMark & 0.987 & 0.518 & 0.9 & 0.2 \\
 & VINE & 0.997 & 0.517 & 0.8 & 0.7 \\
\bottomrule
\end{tabular*}
\end{table}

\noindent\textbf{Post-edit distortions and geometric transformations.}
\cref{fig:pixel_attacks} extends the robustness evaluation to the pixel-space editors, including crops, rotations, and up to three successive edits, with six DiffusionCLIP and twelve Eff-Diff attributes per watermark. For VINE with Eff-Diff, \textit{SafeMark} retains $99.2\%$ detection after regeneration at strength $0.2$, compared with $37.5\%$ for direct editing. Geometric transformations expose weaknesses of the watermarks themselves: after a $2^\circ$ rotation or a crop retaining $95\%$ per side, VINE is almost entirely lost even on unedited inputs, and SleeperMark retains $46.7\%$ and $20.0\%$ of \textit{SafeMark} outputs, respectively. TrustMark tolerates mild geometry better, but regeneration at strength $0.2$ reduces detection of the unedited inputs to $43.3\%$ and of the \textit{SafeMark} outputs to $22.8\%$ under DiffusionCLIP. Preservation during editing therefore cannot add robustness that the watermark itself lacks.

\begin{figure*}[t]
\centering
\includegraphics[width=0.9\textwidth]{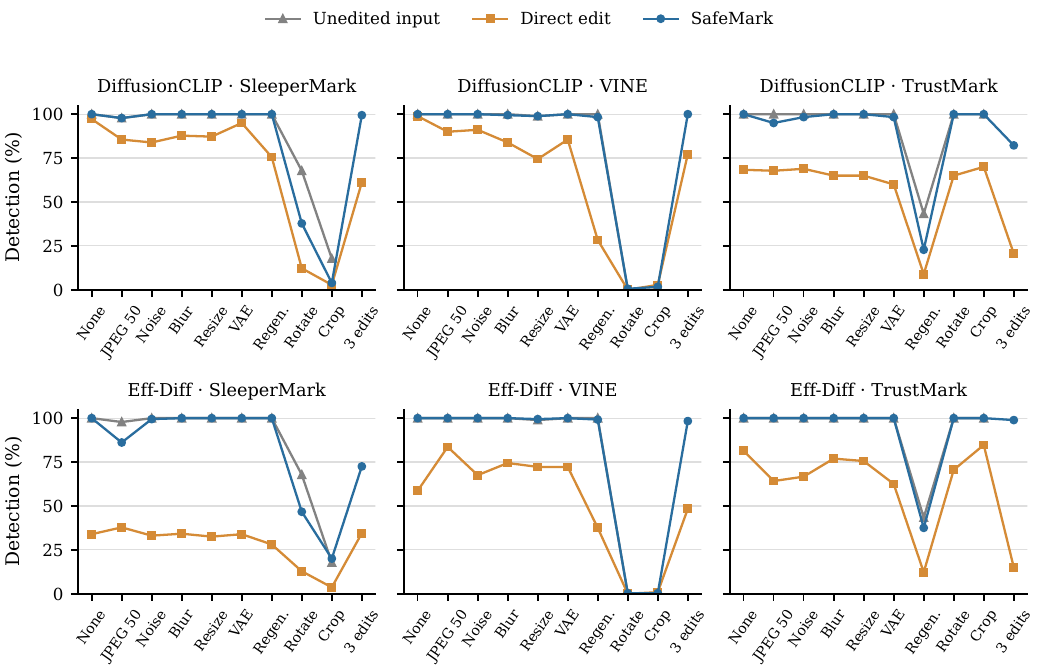}
\caption{\textbf{Detection under post-edit distortions and repeated editing} with the pixel-space editors.}
\label{fig:pixel_attacks}
\end{figure*}

\noindent\textbf{Stable Diffusion editors.}
\cref{tab:sd} reports \textit{SafeMark} with DDIM inversion and InstructPix2Pix on Stable Diffusion v1.5, using held-out images and instructions and three training seeds. VINE and SleeperMark reach full detection with both editors. TrustMark reaches $100\%$ on Bedroom with both editors and $98.9\%$ on Church with DDIM inversion, but only $61.1\%$ on Church with InstructPix2Pix, where the original editor retains $10.0\%$. \cref{fig:sd_examples} shows examples.

\begin{table}[t]
\centering
\caption{\textbf{Stable Diffusion editors} on held-out images and instructions, averaged over three training seeds.}
\label{tab:sd}
\scriptsize
\setlength{\tabcolsep}{2.4pt}
\begin{tabular*}{\columnwidth}{@{\extracolsep{\fill}}lllccccccc@{}}
\toprule
\multirow{2}{*}{Editor} & \multirow{2}{*}{WM} & \multirow{2}{*}{Data} & \multirow{2}{*}{Acc} & \multicolumn{2}{c}{TPR (\%)} & \multicolumn{2}{c}{FPR (\%)} & \multirow{2}{*}{LPIPS} \\
\cmidrule(lr){5-6} \cmidrule(lr){7-8}
 & & & & Base & SM & Base & SM & \\
\midrule
DDIM inv. & SleeperMark & Bedroom & 0.989 & 100.0 & 100.0 & 0.0 & 1.1 & 0.017 \\
 &  & Church & 0.980 & 100.0 & 100.0 & 0.0 & 1.1 & 0.009 \\
\addlinespace[1pt]
 & VINE & Bedroom & 0.986 & 100.0 & 100.0 & 0.0 & 0.0 & 0.019 \\
 &  & Church & 0.942 & 100.0 & 100.0 & 0.0 & 0.0 & 0.022 \\
\addlinespace[1pt]
 & TrustMark & Bedroom & 0.976 & 86.7 & 100.0 & 0.0 & 0.0 & 0.034 \\
 &  & Church & 0.900 & 70.0 & 98.9 & 0.0 & 0.0 & 0.042 \\
\midrule
IP2P & SleeperMark & Bedroom & 0.988 & 100.0 & 100.0 & 3.3 & 2.2 & 0.034 \\
 &  & Church & 0.932 & 76.7 & 100.0 & 3.3 & 1.1 & 0.047 \\
\addlinespace[1pt]
 & VINE & Bedroom & 0.977 & 100.0 & 100.0 & 0.0 & 0.0 & 0.068 \\
 &  & Church & 0.873 & 63.3 & 100.0 & 3.3 & 1.1 & 0.083 \\
\addlinespace[1pt]
 & TrustMark & Bedroom & 0.968 & 72.2 & 100.0 & 0.0 & 2.2 & 0.141 \\
 &  & Church & 0.673 & 10.0 & 61.1 & 0.0 & 1.1 & 0.112 \\
\bottomrule
\end{tabular*}
\end{table}

\begin{figure*}[t]
\centering
\includegraphics[width=0.9\textwidth]{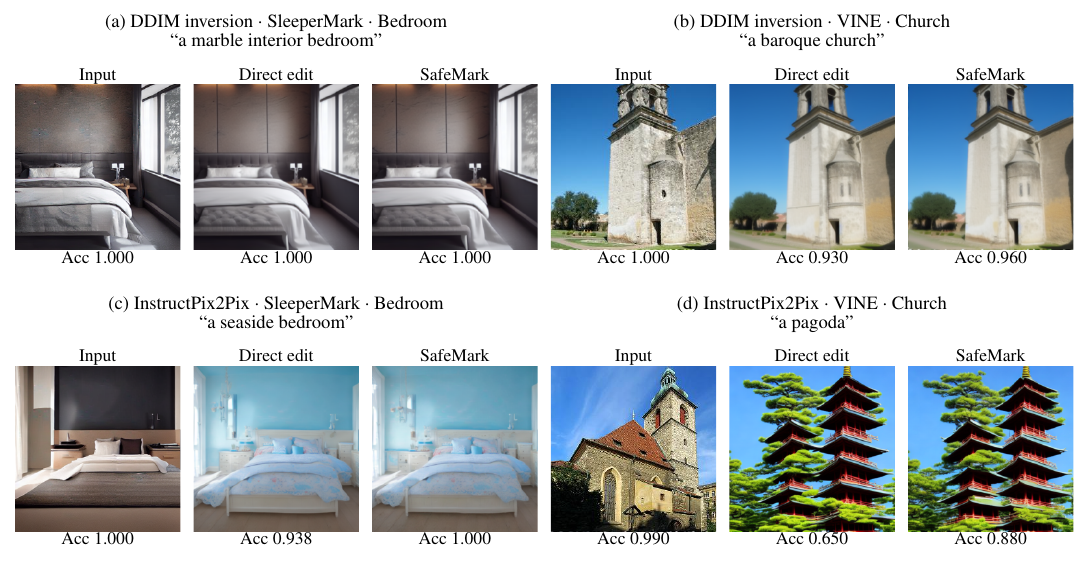}
\caption{\textbf{Stable Diffusion editor examples} on held-out images and instructions. Each triplet shows the watermarked input, the original edit, and the \textit{SafeMark} output.}
\label{fig:sd_examples}
\end{figure*}

\noindent\textbf{Sensitivity to $\tau$.}
\cref{tab:tau} varies the hinge threshold for the pixel-space editors. Lower thresholds reduce bit accuracy without a consistent reduction in LPIPS, while detection stays at $96.7$--$100\%$. We therefore use $\tau=1$, at which the hinge remains active and acts as a mean squared bit error (\cref{sec:penalty}).

\begin{table}[t]
\centering
\caption{\textbf{Sensitivity to the hinge threshold $\tau$} (single runs, pixel-space editors).}
\label{tab:tau}
\scriptsize
\setlength{\tabcolsep}{2.2pt}
\begin{tabular*}{\columnwidth}{@{\extracolsep{\fill}}llcccccccc@{}}
\toprule
\multirow{2}{*}{Editor} & \multirow{2}{*}{WM} & \multicolumn{4}{c}{Acc at $\tau=$} & \multicolumn{4}{c}{LPIPS at $\tau=$} \\
\cmidrule(lr){3-6} \cmidrule(lr){7-10}
 & & 0.8 & 0.9 & 0.95 & 1 & 0.8 & 0.9 & 0.95 & 1 \\
\midrule
DiffCLIP & SleeperMark & 0.960 & 0.921 & 0.979 & 0.994 & 0.093 & 0.047 & 0.074 & 0.058 \\
 & VINE & 0.896 & 0.945 & 0.948 & 0.990 & 0.152 & 0.140 & 0.116 & 0.152 \\
\midrule
Eff-Diff & SleeperMark & 0.919 & 0.933 & 0.925 & 0.985 & 0.029 & 0.022 & 0.021 & 0.027 \\
 & VINE & 0.845 & 0.985 & 0.981 & 0.996 & 0.031 & 0.054 & 0.037 & 0.036 \\
\bottomrule
\end{tabular*}
\end{table}

\section{Comparison with Re-Embedding}\label{app:reembed}

When the editing service holds the watermark encoder, it can decode the input message, apply the original editor, and embed the decoded message into the edited output (Re-embed). \cref{tab:reembed} compares this alternative with \textit{SafeMark} for VINE and TrustMark on matching pixel-space attributes. Re-embed gives higher bit accuracy and much lower LPIPS, for example $0.999$ and $0.005$ against $0.995$ and $0.033$ for VINE with Eff-Diff, so it is the stronger choice when an encoder is available. It has two requirements that \textit{SafeMark} avoids: the service must run the encoder and decoder at inference time, and it writes the decoded bits onto every input, including unwatermarked ones. The column $a_0$ measures the agreement between the bits decoded from an unwatermarked input and from its edited output. Re-embed reproduces these arbitrary bits exactly ($a_0=1.000$), whereas \textit{SafeMark} outputs agree with them at $0.526$--$0.693$. Neither method authenticates the origin of a signal: an input perturbed with projected gradient descent toward the training message at $L_\infty$ budget $2/255$, without the legitimate encoder, is retained by \textit{SafeMark} with $100\%$ detection for SleeperMark and VINE under Eff-Diff, and by Re-embed for VINE and TrustMark. Watermark preservation and watermark authentication are therefore separate problems.

\begin{table}[t]
\centering
\caption{\textbf{\textit{SafeMark} and Re-embed} on matching pixel-space attributes ($N$). $a_0$: agreement between the bits decoded from an unwatermarked input and its edit.}
\label{tab:reembed}
\scriptsize
\setlength{\tabcolsep}{2.2pt}
\begin{tabular*}{\columnwidth}{@{\extracolsep{\fill}}llrcccccccc@{}}
\toprule
\multirow{2}{*}{Editor} & \multirow{2}{*}{WM} & \multirow{2}{*}{$N$} & \multicolumn{4}{c}{\textit{SafeMark}} & \multicolumn{4}{c}{Re-embed} \\
\cmidrule(lr){4-7} \cmidrule(lr){8-11}
 & & & Acc & LPIPS & FPR & $a_0$ & Acc & LPIPS & FPR & $a_0$ \\
\midrule
DiffCLIP & VINE & 1 & 0.989 & 0.123 & 3.3 & 0.693 & 0.999 & 0.015 & 0.0 & 1.000 \\
Eff-Diff & VINE & 4 & 0.995 & 0.033 & 0.0 & 0.660 & 0.999 & 0.005 & 0.0 & 1.000 \\
DiffCLIP & TrustMark & 1 & 0.881 & 0.168 & 0.0 & 0.526 & 1.000 & 0.001 & 0.0 & 1.000 \\
Eff-Diff & TrustMark & 4 & 0.992 & 0.054 & 0.0 & 0.571 & 1.000 & 0.001 & 0.0 & 1.000 \\
\bottomrule
\end{tabular*}
\end{table}

\end{document}